\documentclass[twoside]{article} 
\usepackage[accepted]{aistats2017}
\usepackage{graphicx} % more modern

\usepackage[T1]{fontenc}    % use 8-bit T1 fonts
\usepackage{url}            % simple URL typesetting
\usepackage{booktabs}       % professional-quality tables
\usepackage{amsfonts}       % blackboard math symbols
\usepackage{nicefrac}       % compact symbols for 1/2, etc.
\usepackage{microtype}      % microtypography

\usepackage{amsmath, amsfonts, amssymb, amsthm, dsfont}
\usepackage{algorithm}
\usepackage{algorithmic}
\usepackage{xspace}
\usepackage[numbers]{natbib}
\usepackage{wrapfig}
\usepackage[hang,flushmargin]{footmisc} 
\newlength{\minipagewidth}
\newlength{\minipagewidthx}
\newcommand{\bookboxx}[1]{\small
\par\medskip\noindent
\framebox[0.44\textwidth]{
\begin{minipage}{0.4\dimexpr\textwidth-\parindent\relax} {#1} \end{minipage} } \par\medskip }

\usepackage[colorinlistoftodos, disable,textwidth=20mm]{todonotes}
\definecolor{citrine}{rgb}{0.89, 0.82, 0.04}
\definecolor{blued}{RGB}{70,197,221}
\newcommand{\transp}{\mathsf{T}}
\newcommand{\Tr}{\text{Tr}}

\newcommand{\erf}{\text{erf}}

\newcommand{\distro}{\mathcal{D}^{\ts}}

\newcommand{\rls}{{\small\textsc{RLS}}\xspace}
\newcommand{\ts}{{\small\textsc{TS}}\xspace}
\newcommand{\opt}{\text{opt}\xspace}
\newcommand{\ofulq}{{\small\textsc{OFU-LQ}}\xspace}

\newcommand{\wt}[1]{\widetilde{#1}}
\newcommand{\wh}[1]{\widehat{#1}}

\def\:#1{\protect \ifmmode {\mathbf{#1}} \else {\textbf{#1}} \fi}

\newcommand{\F}{\mathcal F}

\newcommand{\calE}{\mathcal E}

\newcommand{\I}{\mathds{1}}

\renewcommand{\Re}{\mathbb{R}}

\newtheorem{lemma}{Lemma}
\newtheorem{assumption}{Assumption}
\newtheorem{corollary}{Corollary}
\newtheorem{proposition}{Proposition}
\newtheorem{definition}{Definition}
\newtheorem{theorem}{Theorem}

\begin{document}



\twocolumn[

\aistatstitle{Thompson Sampling for Linear-Quadratic Control Problems}

\aistatsauthor{ Marc Abeille \And Alessandro Lazaric }

\aistatsaddress{Inria Lille - Nord Europe, Team SequeL} ]

\begin{abstract}
We consider the exploration-exploitation tradeoff in linear quadratic (LQ) control problems, where the state dynamics is linear and the cost function is quadratic in states and controls. We analyze the regret of Thompson sampling (\ts) (a.k.a. posterior-sampling for reinforcement learning) in the frequentist setting, i.e., when the parameters characterizing the LQ dynamics are fixed. Despite the empirical and theoretical success in a wide range of problems from multi-armed bandit to linear bandit, we show that when studying the frequentist regret \ts in control problems, we need to trade-off the frequency of sampling optimistic parameters and the frequency of switches in the control policy. This results in an overall regret of $O(T^{2/3})$, which is significantly worse than the regret $O(\sqrt{T})$ achieved by the optimism-in-face-of-uncertainty algorithm in LQ control problems.
\end{abstract}

% !TEX root = main.tex

%%%%%%%%%%%%%%%%%%%%%%%%%%%%%%%%%%%%%%%%%%%%%%%%%%%%%%%%%%%%%%%%%%%%%%
%%%%%%%%%%%%%%%%%%%%%%%%%%%%%%%%%%%%%%%%%%%%%%%%%%%%%%%%%%%%%%%%%%%%%%
%%%%%%%%%%%%%%%%%%%%%%%%%%%%%%%%%%%%%%%%%%%%%%%%%%%%%%%%%%%%%%%%%%%%%%

\vspace{-0.2in}
\section{Introduction}
\vspace{-0.1in}

One of the most challenging problems in reinforcement learning (RL) is how to effectively trade off exploration and exploitation in an unknown environment. A number of learning methods has been proposed in finite Markov decision processes (MDPs) and they have been analyzed in the PAC-MDP (see e.g.,~\cite{strehl2009reinforcement}) and the regret framework (see e.g.,~\cite{jaksch2010near-optimal}). The two most popular approaches to address the exploration-exploitation trade-off are the optimism-in-face-of-uncertainty (OFU) principle, where optimistic policies are selected according to upper-confidence bounds on the true MDP paramaters, and the Thompson sampling (\ts) strategy\footnote{In RL literature, \ts has been introduced by~\citet{strens2000a-bayesian} and it is often referred to as posterior-sampling for reinforcement learning (PSRL).}, where random MDP parameters are selected from a posterior distribution and the corresponding optimal policy is executed. Despite their success in finite MDPs, extensions of these methods and their analyses to continuous state-action spaces are still rather limited. \citet{osband2016generalization} study how to randomize the parameters of a linear function approximator to induce exploration and prove regret guarantees in the finite MDP case. \citet{osband2015bootstrapped} develops a specific \ts method applied to the more complex case of neural architectures with significant empirical improvements over alternative exploration strategies, although with no theoretical guarantees. In this paper, we focus on a specific family of continuous state-action MDPs, the linear quadratic (LQ) control problems, where the state transition is linear and the cost function is quadratic in the state and the control. Despite their specific structure, LQ models are very flexible and widely used in practice (e.g., to track a reference trajectory). If the parameter $\theta$ defining dynamics and cost is known, the optimal control can be computed explicitly as a linear function of the state with an appropriate gain. %\footnote{While \textit{exactly} linear systems are quite rare in real-world applications, they are often used in practice to trade off the advantage of having an explicit optimal control and the approximation error introduced by the system linearization.} 
On the other hand, when $\theta$ is unknown, an exploration-exploitation trade-off needs to be solved. \citet{bittanti2006adaptive} and \citet{campi1998adaptive}, first proposed an optimistic approach to this problem, showing that the performance of an adaptive control strategy asymptotically converges to the optimal control. Building on this approach and the OFU principle, \citet{abbasi2011regret} proposed a learning algorithm (\ofulq) with $O(\sqrt{T})$ cumulative regret. \citet{abbasi2015bayesian} further studied how the \ts strategy, could be adapted to work in the LQ control problem. Under the assumption that the true parameters of the model are drawn from a known prior, they show that the so-called Bayesian regret matches the $O(\sqrt{T})$ bound of \ofulq. 

In this paper, we analyze the regret of \ts in LQ problems in the more challenging frequentist case, where $\theta$ is a \text{fixed} parameter, with no prior assumption of its value. The analysis of \ofulq relies on three main ingredients: \textbf{1)} optimistic parameters, \textbf{2)} lazy updates (the control policy is updated only a logarithmic number of times) and \textbf{3)} concentration inequalities for regularized least-squares used to estimate the unknown parameter $\theta$. While we build on previous results for the least-squares estimates of the parameters, points \textbf{1)} and \textbf{2)} should be adapted for \ts. Unfortunately, the Bayesian regret analysis of \ts in~\cite{abbasi2015bayesian} does not apply in this case, since no prior is available on $\theta$. Furthermore, we show that existing frequentist regret analysis for \ts in linear bandit~\cite{agrawal2012thompson} cannot be generalized to the LQ case. This requires deriving a novel line of proof in which we first prove that \ts has a constant probability to sample an optimistic parameter (i.e., an LQ system whose optimal expected average cost is smaller than the true one) and then we exploit the LQ structure to show how being optimistic allows to directly link the regret to the controls operated by \ts over time and eventually bound them. Nonetheless, this analysis reveals a critical trade-off between the frequency with which new parameters are sampled (and thus the chance of being optimistic) and the regret cumulated every time the control policy changes. In \ofulq this trade-off is easily solved by construction: the lazy update guarantees that the control policy changes very rarely and whenever a new policy is computed, it is guaranteed to be optimistic. On the other hand, \ts relies on the \textit{random} sampling process to obtain optimistic models and if this is not done \textit{frequently enough}, the regret can grow unbounded. This forces \ts to favor short episodes and we prove that this leads to an overall regret of order $O(T^{2/3})$ in the one-dimensional case (i.e., both states and controls are scalars), which is significantly worse than the $O(\sqrt{T})$ regret of \ofulq.

%See also~\citep{gopalan2015thompson} for a regret analysis in finite MDPs.
% !TEX root = main.tex

%%%%%%%%%%%%%%%%%%%%%%%%%%%%%%%%%%%%%%%%%%%%%%%%%%%%%%%%%%%%%%%%%%%%%%
%%%%%%%%%%%%%%%%%%%%%%%%%%%%%%%%%%%%%%%%%%%%%%%%%%%%%%%%%%%%%%%%%%%%%%
%%%%%%%%%%%%%%%%%%%%%%%%%%%%%%%%%%%%%%%%%%%%%%%%%%%%%%%%%%%%%%%%%%%%%%

\vspace{-0.1in}
\section{Preliminaries}
\vspace{-0.1in}

\textbf{The control problem.} We consider the discrete-time infinite-horizon linear quadratic (LQ) control problem. Let $x_t\in\Re^n$ be the state of the system and $u_t\in\Re^d$ be the control at time $t$; an LQ problem is characterized by linear dynamics and a quadratic cost function 
\begin{equation}
\begin{aligned}
&x_{t+1} = A_* x_t + B_* u_t + \epsilon_{t+1}, \\
&c(x_t,u_t) = x_t^\transp Q x_t + u_t^\transp R u_t,
\end{aligned}
\label{eq:linear_dynamic_quadratic_cost}
\end{equation}
where $A_*$ and $B_*$ are \textit{unknown} matrices and $Q$ and $R$ are \textit{known} positive definite matrices of appropriate dimension. We summarize the unknown parameters in $\theta_*^\transp = (A_*, B_*)$. The noise process $\epsilon_{t+1}$ is zero-mean and it satisfies the following assumption.

\begin{assumption}\label{asm:noise.asm}
$\{\epsilon_t\}_t$ is a $\mathcal{F}_{t}-$martingale difference sequence, where $\F_t$ is the filtration which represents the information knowledge up to time $t$. 
\end{assumption}

In LQ, the objective is to design a closed-loop control policy $\pi: \Re^n \rightarrow \Re^d$ mapping states to controls that minimizes the average expected cost 
\begin{equation}
J_\pi(\theta_*) = \limsup_{T \rightarrow \infty} \frac{1}{T} \mathbb{E}\bigg[\sum_{t=0}^T c(x_t,u_t)\bigg],
\label{eq:average_cost_criterion}
\end{equation}
with $x_0\!=\!0$ and $u_t \!=\! \pi(x_t)$.
%\citep{bertsekas1995dynamic}
Standard theory for LQ control guarantees that the optimal policy is linear in the state and that the corresponding average expected cost is the solution of a Riccati equation. %In particular, we recall the main result of LQ theory.

\begin{proposition}[Thm.16.6.4 in~\cite{lancaster1995algebraic}]\label{th:lqr}
Under Asm.~\ref{asm:noise.asm} and for any LQ system with parameters $\theta^\transp = (A, B)$ such that $(A,B)$ is stabilizable\footnote{$(A,B)$ is stabilizable if there exists a control gain matrix $K$ s.t. $A + B K$ is stable (i.e., all eigenvalues are in $(-1,1)$).% See~\citep{lancaster1995algebraic} for a complete introduction.
}, and p.d.\ cost matrices $Q$ and $R$, the optimal solution of Eq.~\ref{eq:average_cost_criterion} is given by
\begin{equation}
\begin{aligned}
\pi(\theta) &= K(\theta) x_t,  \quad J(\theta) =\Tr (P(\theta)), \\
K(\theta) &= -(R + B^\transp P(\theta) B)^{-1} B^\transp P(\theta) A , \\
P(\theta) &= Q + A^\transp P(\theta) A +  A^\transp P(\theta) B K(\theta)
%P(\theta) &= Q + A^\transp P(\theta) A - A^\transp P(\theta) B  (R + B^\transp P(\theta) B)^{-1}B^\transp P(\theta) A,
\end{aligned}
\label{eq:lqr_solution}
\end{equation}
where $\pi(\theta)$ is the optimal policy, $J(\theta)$ is the corresponding average expected cost, $K(\theta)$ is the optimal gain, and $P(\theta)$ is the unique solution to the Riccati equation associated with the control problem. Finally, we also have that $A + B K(\theta)$ is asymptotically stable.
\end{proposition}

For notational convenience, we use $H(\theta) = \big(I \; K(\theta)^\transp\big)^\transp$, so that the closed loop dynamics $A + B K(\theta)$ can be equivalently written as $\theta^\transp H(\theta)$.
We introduce further assumptions about the LQ systems we consider.
%Before defining the learning problem, we need to characterize more precisely the problem that we consider and its properties.

\begin{assumption}\label{asm:control.asm} 
We assume that the LQ problem is characterized by parameters $(A_*,B_*,Q, R)$ such that %, thus making $\mathcal{S}$ well defined.}
the cost matrices $Q$ and $R$ are symmetric p.d., and $\theta_* \in \mathcal{S}$ where\footnote{Even if $P(\theta)$ is not defined for every $\theta$, we extend its domain of definition by setting $P(\theta) = + \infty$.} $\mathcal{S} = \{ \theta \in \mathbb{R}^{(n+d) \times n} \text{ s.t. } \Tr (P(\theta)) \leq D \text{ and } \Tr(\theta \theta^\transp) \leq S^2 \}$.
\end{assumption}

While Asm.~\ref{asm:noise.asm} basically guarantees that the linear model in Eq.~\ref{eq:linear_dynamic_quadratic_cost} is correct, Asm.~\ref{asm:control.asm} restricts the control parameters to the admissible set $\mathcal{S}$. This is used later in the learning process and it replaces Asm.~A2-4 in~\cite{abbasi2011regret} in a synthetic way, as shown in the following proposition.

\begin{proposition}\label{prop:admissib_param_set}
Given an admissible set $\mathcal{S}$ as defined in Asm.~\ref{asm:control.asm}, we have \textbf{1)} $\mathcal{S} \subset \{ \theta^\transp = (A,B) \text{ s.t. } $(A,B)$ \text{ is stabilizable}\}$, \textbf{2)} $\mathcal{S}$ is compact, and \textbf{3)} there exists $\rho < 1$ and $C < \infty$ positive constants such that $\rho = \sup_{\theta \in \mathcal{S}} \| A + B K(A,B) \|_2$ and $C = \sup_{\theta \in \mathcal{S}} \| K(\theta) \|_2$.\footnote{We use $\|\cdot\|$ and $\|\cdot\|_2$ to denote the Frobenius and the 2-norm respectively.}.
%\begin{enumerate}
%\item $\mathcal{S} \subset \{ \theta^\transp = (A,B) \text{ s.t. } $(A,B)$ \text{ is stabilizable}\}$,
%\item $\mathcal{S}$ is compact,
%\item there exists $\rho < 1$ and $C < \infty$ positive constants such that $\rho = \sup_{\theta \in \mathcal{S}} \| A + B K(A,B) \|_2$ and $C = \sup_{\theta \in \mathcal{S}} \| K(\theta) \|_2$.\footnote{We use $\|\cdot\|$ and $\|\cdot\|_2$ to denote the Frobenius and the 2-norm respectively.}
%\end{enumerate}
\end{proposition}

As an immediate result, any system with $\theta \in \mathcal{S}$ is stabilizable, and therefore, Asm.~\ref{asm:control.asm} implies that Prop.~\ref{th:lqr} holds.
Finally, we derive a result about the regularity of the Riccati solution, which we later use to relate the regret to the controls performed by \ts.

\begin{lemma}\label{le:riccati.regularity}
Under Asm.~\ref{asm:noise.asm} and for any LQ with parameters $\theta^\transp = (A, B)$ and cost matrices $Q$ and $R$ satisfying Asm.~\ref{asm:control.asm}, let $J(\theta) = \Tr (P(\theta))$ be the optimal solution of Eq.~\ref{eq:average_cost_criterion}. Then, the mapping $\theta \in \mathcal{S} \rightarrow \Tr (P(\theta))$ is continuously differentiable.
Furthermore, let $A_c(\theta) = \theta^\transp H(\theta)$ be the closed-loop matrix, then the directional derivative of $J(\theta)$ in a  direction $\delta \theta$, denoted as $\nabla J(\theta)^\transp \delta\theta$, where $\nabla J(\theta) \in \mathbb{R}^{(n+d)\times n}$ is the gradient of $J$, is the solution of the Lyapunov equation
\begin{equation*}
\nabla J(\theta)^\transp \delta\theta = A_c(\theta)^\transp \nabla J(\theta)^\transp \delta\theta  A_c(\theta)  + C(\theta,\delta\theta) +  C(\theta,\delta\theta)^\transp,
\end{equation*}
where $C(\theta,\delta\theta) = A_c(\theta)^\transp P(\theta) \delta \theta^\transp H(\theta)$.
\end{lemma}

\textbf{The learning problem.}
%Since $\theta_*$ is initially unknown, 
At each time $t$, the learner chooses a policy $\pi_t$, it executes the induced control $u_t = \pi_t(x_t)$ and suffers a cost $c_t = c(x_t,u_t)$. The performance is measured by the cumulative \textit{regret} up to time $T$ as $R_T = \sum_{t=0}^T (c^{\pi_t}_t - J_{\pi_*}(\theta_*))$,
%
%\begin{equation*}
%R_T = \sum_{t=0}^T (c^{\pi_t}_t - J_{\pi_*}(\theta_*)),
%\end{equation*}
%
where at each step the difference between the cost of the controller $c^\pi$ and the expected average cost $J_{\pi_*}(\theta_*)$ of the optimal controller $\pi_*$ is measured.
Let $(u_0,\ldots,u_t)$ be a sequence of controls and $(x_0,x_1,\ldots,x_{t+1})$ be the corresponding states, then $\theta^\star$ can be estimated by regularized least-squares (RLS). Let $z_t = (x_t,  u_t)^\transp$, for any regularization parameter $\lambda \in \mathbb{R}_+^* $, the design matrix and the RLS estimate are defined as 
\begin{align*}
V_{t} = \lambda I + \sum_{s=0}^{t-1} z_s z_s^\transp; \quad \quad\wh{\theta}_{t} = V_{t}^{-1} \sum_{s=0}^{t-1} z_s x_{s+1}^\transp.
\end{align*}
For notational convenience, we use $W_t = V_t^{-1/2}$. We recall a concentration inequality for RLS estimates.
\begin{proposition}[Thm.~2 in~\cite{abbasi-yadkori2011improved}]\label{p:concentration}
We assume that $\epsilon_t$ are conditionally and component-wise sub-Gaussian of parameter $L$ and that 
%, i.e., there exists a known constant $L > 0$ such that for any $\gamma \in \mathbb{R}$ and index $j$, $\mathbb{E} [ \exp (\gamma \epsilon_{t+1,j} ) | \mathcal{F}_t ] \leq \exp ( \gamma^2 L^2 / 2).$ 
%Furthermore, we assume that 
$\mathbb{E} (\epsilon_{t+1} \epsilon_{t+1}^\transp | \mathcal{F}_t) = I$. Then for any $\delta \in (0,1)$ and any $\mathcal{F}_t$-adapted sequence $(z_0,\ldots,z_t)$, the RLS estimator $\hat{\theta}_t$ is such that
\begin{equation}\label{eq:self_normalized2}
\Tr \left( (\hat{\theta}_t - \theta_*)^\transp V_t (\hat{\theta}_t - \theta_*) \right) \leq \beta_t(\delta)^2,
\end{equation}
w.p.\ $1-\delta$ (w.r.t.\ the noise $\{\epsilon_t\}_t$ and any randomization in the choice of the control), where
%\todoaout{Not sure about the dimensionality... shouldn't it be $n+d$?}\todomout{I think the $n$ comes from the fact that the observation is $n$ dimensional (instead of 1D) so $nL$ replaces $R$ in the self normalized bound}
\begin{equation}\label{eq:beta}
\beta_t(\delta) = n L  \sqrt{2 \log \Big( \frac{\det(V_t)^{1/2}}{\det(\lambda I)^{1/2} } \Big)} + \lambda^{1/2} S.
\end{equation}
Further, when $\| z_t \| \leq Z$, 
\begin{equation*}
\frac{\det(V_t)}{\det(\lambda I)} \leq (n+d) \log \big( 1 + T Z^2/\lambda (n+d) \big).
\end{equation*}
\end{proposition}
At any step $t$, we define the ellipsoid $\mathcal{E}^{\rls}_{t} = \big\{ \theta \in \mathbb{R}^d \hspace{1mm} | \hspace{1mm} \| \theta - \hat{\theta}_{t} \|_{V_{t}} \leq \beta_{t}(\delta^\prime) \big\}$
centered in $\wh\theta_t$ with orientation $V_t$ and radius $\beta_t(\delta^\prime)$, with $\delta^\prime = \delta / (4T) $. % is designed to match the confidence level $1 - \delta$ of the final regret bound.\\
Finally, we report a standard result of RLS that, together with Prop.~\ref{p:concentration}, shows that the prediction error on the points $z_t$ used to construct the estimator $\wh\theta_t$ is cumulatively small.
\begin{proposition}[Lem.~10 in~\cite{abbasi2011regret}]\label{p:self_normalized_determinant}
Let $\lambda \geq 1$, for any arbitrary $\mathcal{F}_t$-adapted sequence $(z_0, z_1, \ldots, z_{t})$, let $V_{t+1}$ be the corresponding design matrix, then
\begin{equation}\label{eq:natural_explored_direction_ls}
\sum_{s=0}^{t} \min\big( \|z_s  \|_{V_s^{-1}}^2, 1 \big) \leq 2 \log \frac{\det(V_{t+1})}{\det(\lambda I)}.
\end{equation}

Moreover, when $\|z_t\| \leq Z$ for all $t \geq 0$, then 
\begin{equation*}
\sum_{s=0}^{t} \|z_s \|_{V_s^{-1}}^2 \leq 2 \frac{Z^2}{\lambda} (n + d) \log \Big(1 + \frac{(t+1) Z^2}{\lambda (n+d)} \Big).
\end{equation*}
\end{proposition}
%

% !TEX root = main.tex

%%%%%%%%%%%%%%%%%%%%%%%%%%%%%%%%%%%%%%%%%%%%%%%%%%%%%%%%%%%%%%%%%%%%%%
%%%%%%%%%%%%%%%%%%%%%%%%%%%%%%%%%%%%%%%%%%%%%%%%%%%%%%%%%%%%%%%%%%%%%%
%%%%%%%%%%%%%%%%%%%%%%%%%%%%%%%%%%%%%%%%%%%%%%%%%%%%%%%%%%%%%%%%%%%%%%

\section{Thompson Sampling for LQR}
\label{sec:thompson}

\begin{figure}[t]
\vspace{-5pt}
\begin{center}
\bookboxx{
\begin{small}
    \begin{algorithmic}[1]
        \renewcommand{\algorithmicrequire}{\textbf{Input:}}
        \renewcommand{\algorithmicensure}{\textbf{Output:}}
        \vspace{-0.02in}
        \REQUIRE $\hat{\theta}_0$, $V_0 = \lambda I$, $\delta$, $T$, $\tau$, $t_0 = 0$
        \STATE Set $\delta' = \delta/(8T)$
        \FOR{$t = \{0, \dots, T\}$}
        	    \IF{$\det(V_t) > 2 \det(V_0)$ \textbf{or} $t \geq t_0 + \tau$}
        	    \WHILE{$\wt\theta_t \notin \mathcal{S}$}
                 \STATE Sample $\eta_t \sim \distro$
                 \STATE Compute $\wt\theta_{t} = \wh\theta_{t} + \beta_{t} (\delta') V_{t}^{-1/2} \eta_t$
             \ENDWHILE
             \STATE Let $V_0 = V_t$, $t_0 = t$,
             \ELSE
             \STATE $\wt\theta_t = \wt\theta_{t-1}$
             \ENDIF
             \STATE Execute control             $u_t = K(\wt\theta_t) x_t$
            \STATE   Move to state $x_{t+1}$, receive cost $c(x_t,u_t)$
            \STATE  Compute $V_{t+1}$ and $\wh\theta_{t+1}$% using Eq.~\ref{eq:design.matrix.rls}
        \ENDFOR
    \end{algorithmic}
\vspace{-0.1in}
    \caption{\small Thompson sampling algorithm.}
    \label{alg:ts}
\end{small}
}
\end{center}
\vspace{-0.1in}
\end{figure}

We introduce a specific instance of \ts for learning in LQ problems obtained as a modification of the algorithm proposed in~\cite{abbasi2015bayesian}, where we replace the Bayesian structure and the Gaussian prior assumption with a generic randomized process and we modify the update rule. The algorithm is summarized in Alg.~\ref{alg:ts}.
At any step $t$, given the $\rls$-estimate $\wh\theta_t$ and the design matrix $V_t$, \ts samples a \textit{perturbed} parameter $\wt\theta_t$. In order to ensure that the sampling parameter is indeed admissible, we re-sample it until a valid $\wt\theta_t \in \mathcal{S}$ is obtained. Denoting as $\mathcal{R}_\mathcal{S}$ the rejection sampling operator associated with the admissible set $\mathcal{S}$, we define $\wt\theta_t$ as
\begin{equation}\label{eq:gaussian_sampling}
\wt{\theta}_{t} = \mathcal{R}_\mathcal{S} ( \wh{\theta}_{t} + \beta_{t}(\delta^\prime)W_t  \eta_t),
\end{equation}
where $W_t = V_{t}^{-1/2}$ and every coordinate of the matrix $\eta_t\in\Re^{(n+d)\times(n+d)}$ is a random sample drawn i.i.d.\ from $\mathcal{N}(0,1)$. We refer to this distribution as $\distro$. Notice that such sampling does not need to be associated with an actual posterior over $\theta^\star$ but it just need to randomize parameters coherently with the \rls estimate and the uncertainty captured in $V_t$. 
%
%%
%\begin{definition}\label{def:ts.exploration} [TO MODIFY, EVERYTHING IS WRONG IN d DIMENSION]
%$\distro$ is a multivariate distribution on $\mathbb{R}^{n+d,n}$ absolutely continuous with respect to the Lebesgue measure 
%%\footnote{This is a technical requirement needed later in Lemma~\ref{lem:gradient.optimal.arm}.} 
%which satisfies the following properties:
%\vspace{-0.1in}
%\begin{enumerate}
%\item The associated pdf denoted as $\Phi_\distro$ is log-concave,
%\item $\distro$ is rotationally invariant,
%\item there exists a strictly positive probability $p$ such that for any $u \in \mathbb{R}^d$ with $\|u\| = 1$, 
%$$\mathbb{P}_{\eta \sim \distro} \big(u^\top \eta  \geq 1 \big) \geq p, \quad\quad \text{\textit{(anti-concentration inequality),}}$$
%\vspace{-0.3in}
%\item there exists $c,c^\prime$ positive constant such that $\forall \delta \in (0,1)$ 
%$$\mathbb{P}_{\eta \sim \distro} \bigg( \|\eta\| \leq \sqrt{c d \log \frac{c^\prime d}{\delta} } \bigg) \geq 1 - \delta,  \quad\quad \text{\textit{(concentration inequality),}}$$
%\vspace{-0.15in}
%\end{enumerate} 
%\end{definition}
%%
%
Let $\gamma_t(\delta) = \beta_t(\delta^\prime)n \sqrt{2 (n+d) \log\big( 2 n (n+d) / \delta}\big)$, then the high-probability \ts ellipsoid $\calE^{\ts}_{t} = \{ \theta \in \mathbb{R}^d \hspace{1mm} | \hspace{1mm} \| \theta - \wh{\theta}_{t} \|_{V_{t}} \leq \gamma_{t}(\delta^\prime)\}$ is defined so that any parameter $\wt\theta_t$ belongs to it with $1-\delta/8$ probability.
%
%The difference between $\calE^{\rls}_{t}$ and $\calE_t^\ts$ lies in the additional factor $\sqrt{d}$ in the definition of $\gamma_t(\delta)$. This extra-sampling is here to ensure that the sampling covers the $\calE^{\rls}_{t}$ ellipsoid which represents the uncertainty. This is encoded into the fixed probability of being optimistic (see Lem.~\ref{le:probability_optimistic}).

%In Sect.~\ref{sec:proof} we prove that any distribution satisfying the conditions in Def.~\ref{def:ts.exploration} introduces the right amount of randomness to achieve the desired regret without actually satisfying any Bayesian assumption. Def.~\ref{def:ts.exploration} includes the Gaussian prior as well as other distributions such as uniform on the unit ball $\mathcal{B}_d(0,\sqrt{d})$ or distributions concentrated on the boundary of $\calE^\ts_t$ (refer to App.~\ref{sec:app_examples} for exact values of $c$, $c'$, and $p$ for uniform and Gaussian distributions).

Given the parameter $\wt\theta_t$, the gain matrix $K(\wt\theta_t)$ is computed and the corresponding optimal control $u_t = K(\wt\theta_t) x_t$ is applied. As a result, the learner observes the cost $c(x_t,u_t)$ and the next state $x_{t+1}$, and $V_t$ and $\wh\theta_t$ are updated accordingly. Similar to most of RL strategies, the updates are not performed at each step and the same estimated optimal policy $K(\wt\theta_t)$ is kept constant throughout an \textit{episode}. Let $V_0$ be the design matrix at the beginning of an episode, then the episode is terminated upon two possible conditions: \textbf{1)} the determinant condition of the design matrix is doubled (i.e., $\det(V_t) \geq 2 \det(V_0)$) or \textbf{2)} a maximum length condition is reached. While the first condition is common to all RL strategies, here we need to force the algorithm to interrupt episodes as soon as their length exceeds $\tau$ steps. The need for this additional termination condition is intrinsically related to the \ts nature and it is discussed in detail in the next section.

%(including OFU~\citet{abbasi2011regret}), 
% !TEX root = main.tex

%%%%%%%%%%%%%%%%%%%%%%%%%%%%%%%%%%%%%%%%%%%%%%%%%%%%%%%%%%%%%%%%%%%%%%
%%%%%%%%%%%%%%%%%%%%%%%%%%%%%%%%%%%%%%%%%%%%%%%%%%%%%%%%%%%%%%%%%%%%%%
%%%%%%%%%%%%%%%%%%%%%%%%%%%%%%%%%%%%%%%%%%%%%%%%%%%%%%%%%%%%%%%%%%%%%%
\vspace{-0.1in}
\section{Theoretical analysis}
\label{sec:proof}
\vspace{-0.1in}

We prove the first frequentist regret bound for \ts in LQ systems of dimension $2$ ($n=1$, $d=1$). In order to isolate the steps which explicitly rely on this restriction, whenever possible we derive the proof in the general $n+d$-dimensional case. %We prove the following result.

\begin{theorem}\label{th:main.th}
Consider the LQ system in Eq.~\ref{eq:linear_dynamic_quadratic_cost} of dimension $n=1$ and $d=1$. Under Asm.~\ref{asm:noise.asm} and~\ref{asm:control.asm} for any $0 < \delta < 1$, the cumulative regret of \ts over $T$ steps is bounded w.p. at least $1-\delta$ as~\footnote{Further details can be recovered from the proof.}
$$R(T) = \wt{O}\left( T^{2/3} \sqrt{\log (1/\delta) } \right).$$
%where $\wt{O}$ hides problem-dependent and numerical constants.
\end{theorem}

This result is in striking contrast with previous results in multi-armed and linear bandit where the frequentist regret of \ts is $O(\sqrt{T})$ and the Bayesian analysis of \ts in control problems where the regret is also $O(\sqrt{T})$. As discussed in the introduction, the frequentist regret analysis in control problems introduces a critical trade-off between the frequency of selecting optimistic models, which guarantees small regret in bandit problems, and the reduction of the number of policy switches, which leads to small regret in control problems. Unfortunately, this trade-off cannot be easily balanced and this leads to a final regret of $O(T^{2/3})$. Sect.~\ref{subsec:challenges} provides a more detailed discussion on the challenges of bounding the frequentist regret of \ts in LQ problems. 

%%%%%%%%%%%%%%%%%%%%%%%%%%%%%%%%%%%%%%%%%%%%%%%%%%%%%%%%%%%%%%%%%%%%
%%%%%%%%%%%%%%%%%%%%%%%%%%%%%%%%%%%%%%%%%%%%%%%%%%%%%%%%%%%%%%%%%%%%
%%%%%%%%%%%%%%%%%%%%%%%%%%%%%%%%%%%%%%%%%%%%%%%%%%%%%%%%%%%%%%%%%%%%

\vspace{-0.05in}
\subsection{Setting the Stage}
\vspace{-0.05in}

\textbf{Concentration events.} We introduce the following high probability events.
\begin{definition}\label{def:concentration_events}
Let $\delta\in(0,1)$ and $\delta'=\delta/(8T)$ and $t\in[0,T]$. We define the event (RLS estimate concentration) $\wh{E}_t = \big\{ \forall s \leq t, \hspace{2mm} \|\wh{\theta}_{s} - \theta^\star\|_{V_s} \leq \beta_s(\delta') \big\}$ and the event (parameter $\wt{\theta}_s$ concentrates around $\wh\theta_s$) $\wt{E}_t = \big\{ \forall s \leq t, \hspace{2mm} \|\wt{\theta}_{s} - \wh{\theta}_{s}\|_{V_s} \leq \gamma_s(\delta') \big\}$.
%where $\gamma_s(\delta') = \beta_s(\delta') \sqrt{c d \log(c^\prime d / \delta'))}$.
\end{definition}
We also introduce a high probability event on which the states $x_t$ are bounded almost surely.
\begin{definition}\label{def:bounded.states}
Let $\delta\in(0,1)$, $X,X^\prime$ be two problem dependent positive constants and $t\in[0,T]$. We define the event (bounded states) $\bar{E}_t =  \big\{ \forall s \leq t, \hspace{2mm} \|x_s\|\leq X\log\frac{X^\prime}{\delta} \big\}$.
\end{definition}
Then we have that $\wh{E} := \wh{E}_{T} \subset \dots \subset \wh{E}_{1}$, $\wt{E} := \wt{E}_{T} \subset \dots \subset \wt{E}_{1}$ and  $\bar{E} := \bar{E}_{T} \subset \dots \subset \bar{E}_{1}$. We show that these events do hold with high probability.

\begin{lemma}\label{le:high_proba_concentration} 
$\mathbb{P}(\wh{E} \cap \wt{E}) \geq 1 - \delta/4$.
\end{lemma}

\begin{corollary}\label{co:high_proba_boundedness}
On $\wh{E} \cap \wt{E}$, $\mathbb{P}(\bar{E} ) \geq 1 - \delta/4$. Thus, $\mathbb{P}(\wh{E} \cap \wt{E} \cap \bar{E}) \geq 1 - \delta/2$.
\end{corollary}

Lem.~\ref{le:high_proba_concentration} leverages Prop.~\ref{p:concentration} and the sampling distribution $\distro$ to ensure that $\wh{E} \cap \wt{E}$ holds w.h.p. Furthermore, Corollary~\ref{co:high_proba_boundedness} ensures that the states remains bounded w.h.p. on the events $\wh{E} \cap \wt{E}$.\footnote{This non-trivial result is directly collected from the bounding-the-state section of~\cite{abbasi2011regret}.}
As a result, the proof can be derived considering that both parameters concentrate and that states are bounded, which we summarize in the sequence of events $E_t = \wh{E}_t \cap \wt{E}_t \cap \bar{E}_t$, which holds with probability at least $1- \delta/2$ for all $t\in[0,T]$.

\textbf{Regret decomposition.} Conditioned on the filtration $\F_t$ and event $E_t$, we have $\theta^\star\in\mathcal{E}_t^\rls$, $\wt\theta_t \in \mathcal{E}_t^\ts$ and $\|x_t\| \leq X$. We directly decompose the regret and bound it on this event as~\cite[Sect.~4.2]{abbasi2011regret}
\begin{equation}\label{eq:regret.decomposition}
\begin{aligned}
R(T) =& \underbrace{\sum_{t=0}^T \big \{ J(\wt\theta_t) - J(\theta_*) \big \}\I\{E_t\} }_{R^\ts}\\
 &+ \underbrace{(R^\rls_1 + R^\rls_2 + R^\rls_3)\I\{E_t\}}_{R^\rls}
\end{aligned} 
\end{equation}
where $R^\rls$ is decomposed into the three components
\begin{equation*}
\begin{aligned}
R^\rls_1 &= \sum_{t=0}^T \big\{ \mathbb{E} ( x_{t+1}^\transp P(\wt\theta_{t+1}) x_{t+1}|\mathcal{F}_{t}) - x_t^\transp P(\wt\theta_t) x_t  \big\}, \\ 
R^\rls_2 &= \sum_{t=0}^T\mathbb{E} \big[ x_{t+1}^\top (P(\wt\theta_t)  - P(\wt\theta_{t+1} ) ) x_{t+1} |\mathcal{F}_{t} \big], \\
R^\rls_3 &=  \sum_{t=0}^T \big\{ z_t^\transp \wt\theta_t P(\wt\theta_t) \wt\theta_t^\transp z_t  - z_t^\top \theta_* P(\wt\theta_t) \theta_*^\transp z_t  \big\}.
\end{aligned}
\end{equation*}
Before entering into the details of how to bound each of these components, in the next section we discuss what are the main challenges in bounding the regret.

\vspace{-0.05in}
\subsection{Related Work and Challenges}\label{subsec:challenges}
\vspace{-0.05in}

%We focus on the regret decomposition in Eq.~\ref{eq:regret.decomposition}. 
Since the \rls estimator is the same in both \ts and OFU, the regret terms $R^\rls_1$ and $R^\rls_3$ can be bounded as in~\cite{abbasi2011regret}. In fact, $R^\rls_1$ is a martingale by construction and it can be bounded by Azuma inequality. The term $R^\rls_3$ is related to the difference between the \textit{true} next expected state $\theta_\star^\transp z_t$ and the \textit{predicted} next expected state $\wt\theta_t^\transp z_t$. A direct application of RLS properties makes this difference small by construction, thus bounding $R^\rls_3$. Finally, the $R^\rls_2$ term is directly affected by the changes in model from any two time instants (i.e., $\wt\theta_t$ and $\wt\theta_{t+1}$), while $R^\ts$ measures the difference in optimal average expected cost between the true model $\theta_*$ and the sampled model $\wt\theta_t$. In the following, we denote by $R^\rls_{2,t}$ and $R^\ts_t$ the elements at time $t$ of these two regret terms and we refer to them as \textit{consistency regret} and \textit{optimality regret} respectively.

\textbf{Optimistic approach.} \ofulq explicitly bounds both regret terms directly by construction. In fact, the lazy update of the control policy allows to set to zero the consistency regret $R^\rls_{2,t}$ in all steps but when the policy changes between two episodes. Since in \ofulq an episode terminates only when the determinant of the design matrix is doubled, it is easy to see that the number of episodes is bounded by $O(\log(T))$, which bounds $R^\rls_2$ as well (with a constant depending on the bounding of the state $X$ and other parameters specific of the LQ system).\footnote{Notice that the consistency regret is not specific to LQ systems but it is common to all regret analyses in RL (see e.g., UCRL~\cite{jaksch2010near-optimal}) except for episodic MDPs and it is always bounded by keeping under control the number of switches of the policy (i.e., number of episodes).}% 2 D X^2 (n+d) \log(1 + T(1+C^2)X^2/\lambda)$\todoaout{Write the proper constant.}.
At the same time, at the beginning of each episode an optimistic parameter $\wt\theta_t$ is chosen, i.e., $J(\wt\theta_t) \leq J(\theta_*)$, which directly ensures that $R^\ts_t$ is upper bounded by 0 at each time step.

\textbf{Bayesian regret.} The lazy PSRL algorithm in~\cite{abbasi2015bayesian} has the same lazy update as OFUL and thus it directly controls $R^\rls_{2}$ by a small number of episodes. On the other hand, the random choice of $\wt\theta_t$ does not guarantee optimism at each step anymore. Nonetheless, the regret is analyzed in the Bayesian setting, where $\theta_*$ is drawn from a known prior and the regret is evaluated \textit{in expectation} w.r.t.\ the prior. Since $\wt\theta_t$ is drawn from a posterior constructed from the same prior as $\theta_*$, in expectation its associated $J(\wt\theta_t)$ is the same as $J(\theta_*)$, thus ensuring that $\mathbb{E}[R^\ts_t]=0$.

\textbf{Frequentist regret.} When moving from Bayesian to frequentist regret, this argument does not hold anymore and the (positive) deviations of $J(\wt\theta_t)$ w.r.t. $J(\theta_*)$ has to be bounded in high probability. \citet{abbasi2011regret} exploits the linear structure of LQ problems to reuse arguments originally developed in the linear bandit setting. Similarly, we could leverage on the analysis of \ts for linear bandit by~\citet{agrawal2012thompson} to derive a frequentist regret bound. \citet{agrawal2012thompson} partition the (potentially infinite) arms into \textit{saturated} and \textit{unsaturated} arms depending on their estimated value and their associated uncertainty (i.e., an arm is saturated when the uncertainty of its estimate is smaller than its performance gap w.r.t.\ the optimal arm). In particular, the uncertainty is measured using confidence intervals derived from a concentration inequality similar to Prop.~\ref{p:concentration}. This suggests to use a similar argument and classify policies as saturated and unsaturated depending on their value. Unfortunately, this proof direction cannot be applied in the case of LQR. In fact, in an LQ system $\theta$ the performance of a policy $\pi$ is evaluated by the function $J_\pi(\theta)$ and the policy uncertainty should be measured by a confidence interval constructed as $|J_\pi(\theta_*) - J_\pi(\wh\theta_t)|$. Despite the concentration inequality in Prop.~\ref{p:concentration}, we notice that neither $J_\pi(\theta_*)$ nor $J_\pi(\wh\theta_t)$ may be finite, since $\pi$ may not stabilize the system $\theta_*$ (or $\wh\theta_t$) and thus incur an infinite cost. As a result, it is not possible to introduce the notion of saturated and unsaturated policies in this setting and another line of proof is required. Another key element in the proof of~\cite{agrawal2012thompson} for \ts in linear bandit is to show that \ts has a constant probability $p$ to select optimistic actions and that this contributes to reduce the regret of any non-optimistic step. In our case, this translates to requiring that \ts selects a system $\wt\theta_t$ whose corresponding optimal policy is such that $J(\wt\theta_t) \leq J(\theta_*)$. Lem.~\ref{le:probability_optimistic} shows that this happens with a constant probability $p$. Furthermore, we can show that optimistic steps reduce the regret of non-optimistic steps, thus effectively bounding the optimality regret $R^\ts$. Nonetheless, this is not compatible with a small consistency regret. In fact, we need optimistic parameters $\wt\theta_t$ to be sampled \textit{often enough}. On the other hand, bounding the consistency regret $R^\rls_{2}$ requires to reduce the switches between policies as much as possible (i.e., number of episodes). If we keep the same number of episodes as with the lazy update of OFUL (i.e., about $\log(T)$ episodes), then the number of sampled points is as small as $T/(T-\log(T))$. While \ofulq guarantees that any policy update is optimistic by construction, with \ts, only a fraction $T/(p(T-\log(T))$ of steps would be optimistic \textit{on average}. Unfortunately, such small number of optimistic steps is no longer enough to derive a bound on the optimality regret $R^\ts$. Summarizing, in order to derive a frequentist regret bound for \ts in LQ systems, we need the following ingredient \textbf{1)} constant probability of optimism, \textbf{2)} connection between optimism and $R^\ts$ without using the saturated and unsaturated argument, \textbf{3)} a suitable trade-off between lazy updates to bound the consistency regret and frequent updates to guarantee small optimality regret.

%%%%%%%%%%%%%%%%%%%%%%%%%%%%%%%%%%%%%%%%%%%%%%%%%%%%%%%%%%%%%%%%%%%%%
%%%%%%%%%%%%%%%%%%%%%%%%%%%%%%%%%%%%%%%%%%%%%%%%%%%%%%%%%%%%%%%%%%%%%
%%%%%%%%%%%%%%%%%%%%%%%%%%%%%%%%%%%%%%%%%%%%%%%%%%%%%%%%%%%%%%%%%%%%%

% !TEX root = main.tex

%%%%%%%%%%%%%%%%%%%%%%%%%%%%%%%%%%%%%%%%%%%%%%%%%%%%%%%%%%%%%%%%%%%%%%
%%%%%%%%%%%%%%%%%%%%%%%%%%%%%%%%%%%%%%%%%%%%%%%%%%%%%%%%%%%%%%%%%%%%%%
%%%%%%%%%%%%%%%%%%%%%%%%%%%%%%%%%%%%%%%%%%%%%%%%%%%%%%%%%%%%%%%%%%%%%%

\subsection{Bounding the Optimality Regret $R^\ts$}\label{subsec:Rts.bound}

\textbf{$R^\ts$ decomposition.} We define the ``extended'' filtration $\F_{t}^x = (\F_{t-1}, x_t)$. Let $K$ be the (random) number of episodes up to time $T$, $\{t_k\}_{k=1}^K$ be the steps when the policy is updated, i.e., when a new parameter $\tilde{\theta}$ is sampled, and let $T_k$ be the associated length of each episode, then we can further decompose $R^\ts$ as
\begin{equation}\label{eq:regret_ts}
\begin{aligned}
R^\ts &= \sum_{k=0}^K  T_k  \underbrace{\Big( J(\wt\theta_{t_k})- \mathbb{E}[J(\wt\theta_{t_k})| \mathcal{F}_{t_k}^x,E_{t_k}] \Big) \I_{E_{t_k}} }_{R^{\ts,1}_{t_k}}  \\
&+   \sum_{k=0}^K T_k  \underbrace{ \big \{ \mathbb{E}[J(\wt\theta_{t_k}) | \mathcal{F}_{t_k}^x,E_{t_k}] - J(\theta_*) \big\} \I_{E_{t_k}} }_{R^{\ts,2}_{t_k}}.
\end{aligned}
\end{equation}
%
%We notice that $R^{\ts,1}_{t_k}$ is a martingale difference sequence w.r.t.\ $\mathcal{F}_{t_k-1}$ and thus we can bound its cumulative sum by Azuma inequality. We postpone this step to Sect.~\ref{subsec:final.bound} (some extra care is needed to consider the randomness of $T_k$). 
We focus on the second regret term that we redefine $R^{\ts,2}_{t_k}=\Delta_t$ for any $t = t_k$ for notational convenience.

%\todoa{What this should be about?}
%\begin{proposition}\label{p:R.ts.bound}
%On the event $E$, with probability at least $1 - \delta/2$, .... We can collect every terms and apply Azuma on all of them, or doing it separately (clearer but worse bound)...
%\end{proposition}

\textbf{Optimism and expectation.} 
Let $\Theta^\opt = \{ \theta : J(\theta) \leq J(\theta_*) \}$ be the set of optimistic parameters (i.e., LQ systems whose optimal average expected cost is lower than the true one). Then, for any $\theta \in \Theta^\opt$, the per-step regret $\Delta_t$ is bounded by:
\begin{equation*}
\begin{aligned}
\Delta_t  &\leq \big( \mathbb{E}[J(\wt\theta_t) | \mathcal{F}_{t}^x,E_t] - J(\theta) \big) \I_{E_t} , \\
 &\leq  \Big| J(\theta)  - \mathbb{E}[J(\wt\theta_t)| \mathcal{F}_{t}^x,E_t]\Big| \I_{E_t}, \text{ which implies that} \\
% &\leq  \mathbb{E} \bigg[  \Big| J(\wt\theta)  - \mathbb{E}[J(\wt\theta_t) | \mathcal{F}_{t}^x,E_t] \Big| \I_{E_t}   \mid \mathcal{F}_{t}^x , \wt\theta \in \Theta^\opt \bigg], \\
 \Delta_t &\leq  \mathbb{E} \Big[  \big| J(\wt\theta) \! -\! \mathbb{E}[J(\wt\theta_t) | \mathcal{F}_{t}^x,E_t] \big| \I_{\wt E_t} \!\mid\! \mathcal{F}_{t}^x \!,\! \wh E_t,\! \bar{E}_t,\! \wt\theta \!\in\! \Theta^\opt\! \Big],
 \end{aligned}
\end{equation*}
where we use first the definition of the optimistic parameter set, then bounding the resulting quantity by its absolute value, and finally switch to the expectation over the optimistic set, since the inequality is true for any $\wt\theta \in \Theta^{\opt}$. While this inequality is true for any sampling distribution, it is convenient to select it equivalent to the sampling distribution of \ts. Thus, we set $\wt\theta = \mathcal{R}_\mathcal{S}(\wh\theta_t + \beta_t(\delta^\prime) W_t \eta)$ with $\eta$ is component wise Gaussian $\mathcal{N}(0,1)$ and obtain

\vspace{-0.1in}
\begin{small}
\begin{equation*}
\begin{aligned}
&\Delta_t \leq \mathbb{E} \Big[  \big| J(\wt\theta_t)  - \mathbb{E}[J(\wt\theta_t) | \mathcal{F}_{t}^x,E_t] \big| \I_{\wt E_t}   \mid \mathcal{F}_{t}^x , \wh E_t, \bar{E}_t, \wt\theta_t \in \Theta^\opt \Big], \\
&\leq \frac{\mathbb{E} \Big[  \big| J(\wt\theta_t)  - \mathbb{E}[J(\wt\theta_t) | \mathcal{F}_{t}^x,E_t] \big| \I_{\wt E_t}   \mid \mathcal{F}_{t}^x , \wh E_t, \bar{E}_t \Big]}{\mathbb{P}\big( \wt\theta_t \in \Theta^{\opt} \hspace{1mm} | \hspace{1mm} \F_{t}^x, \wh E_t \big)}.
\end{aligned}
\end{equation*} 
\end{small}
\vspace{-0.1in}

At this point we need to show that the probability of sampling an optimistic parameter $\wt\theta_t$ is constant at any step $t$. This result is proved in the following lemma.

\begin{lemma}\label{le:probability_optimistic}
Let $\Theta^{\opt} := \{ \theta \in\Re^d \hspace{1mm} | \hspace{1mm} J(\theta) \leq J(\theta^\star) \}$ be the set of optimistic parameters and $\wt\theta_t = \mathcal{R}_\mathcal{S} (\wh\theta_t + \beta_t(\delta^\prime) W_t \eta)$ with $\eta$ be component-wise normal $\mathcal{N}(0,1)$, then in the one-dimensional case ($n\!=\!1$ and $d\!=\!1$)
$$\forall t\geq 0, \hspace{1mm}\mathbb{P}\big( \wt\theta_t \in \Theta^{\opt} \hspace{1mm} | \hspace{1mm} \F^x_{t}, \wh E_t \big) \geq p,$$
where $p$ is a strictly positive constant.
\end{lemma}

Integrating this result into the previous expression gives
\begin{equation}\label{eq:variance.bound}
\Delta_t  \leq  \frac{1}{p} \mathbb{E} \bigg[  \Big| J(\wt\theta_t)  - \mathbb{E}[J(\wt\theta_t) | \mathcal{F}_{t}^x, E_t] \Big| \mid \mathcal{F}_{t}^x , E_t \bigg].
\end{equation}
The most interesting aspect of this result is that the constant probability of being optimistic allows us to bound the worst-case non-stochastic quantity $\mathbb{E}[J(\wt\theta_t)| \mathcal{F}_{t}^x] - J(\theta_*)$ depending on $J(\theta_*)$ by an expectation $\mathbb{E} \big[  \big| J(\wt\theta_t)  - \mathbb{E}[J(\wt\theta_t) | \mathcal{F}_{t}^x] \big| \mid \mathcal{F}_{t}^x \big]$ up to a multiplicative constant (we drop the events $E$ for notational convenience). The last term is the conditional \textit{absolute deviation} of the performance $J$ w.r.t.\ the \ts distribution. This connection provides a major insight about the functioning of \ts, since it shows that \ts does not need to have an accurate estimate of $\theta_*$ but it should rather reduce the estimation errors of $\theta_*$ only on the directions that may translate in larger errors in estimating the objective function $J$. In fact, we show later that at each step \ts chooses a sampling distribution that tends to minimize the expected absolute deviations of $J$, thus contributing to reduce the deviations in $R^\ts_t$.

\textbf{Variance and gradient.} Let $d' = \sqrt{n( n + d)}$, we introduce the mapping $f_t$ from the ball $\mathcal{B}(0,d')$ to $\mathbb{R}_+$ defined as 
$$f_t(\eta) = J(\wh\theta_t + \beta_t(\delta^\prime) W_t \eta) - \mathbb{E}[J(\wt\theta_t) | \mathcal{F}_{t}^x, E_t]$$
where the restriction on the ball is here to meet the $\calE^\ts_t$ confidence ellipsoid of the sampling.
Since the perturbation $\eta \sim \distro$ is independent of the past, we can rewrite Eq.~\ref{eq:variance.bound} as
\begin{equation*}
\Delta_t \leq \mathbb{E}_{\eta \sim \distro} \big[ |f_t(\eta)| \big |\eta \in \mathcal{B}(0,d'), \wh\theta_t + \beta_t(\delta^\prime) W_t \eta \in\mathcal{S} \big].
\end{equation*}
We now need to show that this formulation of the regret is strictly related to the policy executed by \ts. We prove the following result (proof in the supplement).

\begin{lemma}
\label{th:poincare.weigthed}
Let $\Omega \subset \mathbb{R}^d$ be a convex domain with finite diameter $\text{\emph{diam}}$. Let $p$ be a non-negative log-concave function on $\Omega$ with continuous derivative up to the second order. Then, for all $u \in W^{1,1}(\Omega)$\footnote{$W^{1,1}(\Omega)$ is the Sobolev space of order 1 in $L^1(\Omega)$.} such that 
$\int_\Omega u(z) p(z) dz = 0$ one has
\begin{equation*}
\int_{\Omega} |f(z)| p(z) dz \leq 2 \text{\emph{diam}} \int_{\Omega} ||\nabla f(z)|| p(z) dz
\end{equation*}
\end{lemma}

%\footnote{$W^{1,1}(\Omega)$ is the Sobolev space of $L^1(\Omega)$.}

Before using the previous result, we relate the gradient of $f_t$ to the gradient of $J$.
Since for any $\eta$ and any $\theta = \wh\theta_t + \beta_t(\delta^\prime) W_t \eta$, we have
\begin{equation*}
\begin{aligned}
\nabla f_t(\eta) = \beta_t(\delta^\prime) W_t \nabla J(\theta)
\end{aligned}
\end{equation*}
To obtain a bound on the norm of $\nabla f_t$, we apply Prop.~\ref{p:gradient.inequality} (derived from Lem.~\ref{le:riccati.regularity}) to get a bound on $\| \nabla J(\theta) \|_{W_t^2}$:
\begin{align*}
\| \nabla J(\theta) \|_{W_t^2} \leq& \|A_c(\theta)\|_2^2 \| \nabla J(\theta) \|_{W_t^2} \\
&+ 2  \|P(\theta)\| \|A_c(\theta)\|_2 \| H(\theta)\|_{W_t^2}.
\end{align*}
Making use of $\| M\| \leq \Tr(M)$ for any positive definite matrix together with $\Tr (P(\theta)) \leq D$ (Asm.~\ref{asm:control.asm}) and $\|A_c(\theta)\|_2 \leq \rho$ (Prop.~\ref{prop:admissib_param_set}), 
\begin{equation*}
\| \nabla J(\theta) \|_{W_t^2} \leq \rho^2 \| \nabla J(\theta) \|_{W_t^2} + 2 D \rho \| H(\theta)\|_{W_t^2},
\end{equation*}
which leads to %$ \| \nabla J(\theta) \|_{W_t^2} \leq 2 D \rho/ (1 -  \rho^2) \| H(\theta)\|_{W_t^2}$.
\begin{equation*}
\| \nabla J(\theta) \|_{W_t^2} \leq 2 D \rho/ (1 -  \rho^2) \| H(\theta)\|_{W_t^2}.
\end{equation*}
We are now ready to use the weighted Poincar\'{e} inequality of Lem.~\ref{th:poincare.weigthed} to link the expectation of $|f_t|$ to the expectation of its gradient. From Lem.~\ref{le:riccati.regularity}, we have $f_t \in W^{1,1}(\Omega)$ and its expectation is zero by construction. On the other hand, the rejection sampling procedure impose that we conditioned the expectation with $ \wh\theta_t + \beta_t(\delta^\prime) W_t \eta \in\mathcal{S} $ which is unfortunately not convex. However, we can still apply Lem.~\ref{th:poincare.weigthed} considering the function $\tilde{f}_t(\eta) = f_t(\eta) \mathds{1} (  \wh\theta_t + \beta_t(\delta^\prime) W_t \eta \in\mathcal{S} )$ and diameter $\text{\emph{diam}} = d'$. As a result, we finally obtain
\begin{equation*}
\begin{aligned}
\Delta_t &\leq \gamma \mathbb{E} \Big[ \big \|  H(\wt\theta_t) \big \|_{W_t^2}|\mathcal{F}_{t}^x \Big],
\end{aligned}
\end{equation*}
where $\gamma = 8 \sqrt{n (n+d)} \beta_T(\delta^\prime) D \rho / (p(1 - \rho^2))$.

\textbf{From gradient to actions.} 
Recalling the definition of $H(\theta) = \big(I \; K(\theta)^\transp\big)^\transp$ we notice that the previous expression bound the regret $\Delta_t$ with a term involving the gain $K(\theta)$ of the optimal policy for the sampled parameter $\theta$. This shows that the $R^{\ts}$ regret is directly related to the policies chosen by \ts. To make such relationship more apparent, we now elaborate the previous expression to reveal the sequence of state-control pairs $z_t$ induced by the policy with gain $K(\wt\theta_t)$. We first plug the bound on $\Delta_t$ back into Eq.~\ref{eq:regret_ts} as
\begin{equation*}
\begin{aligned}
R^{\ts} \leq& \sum_{k=1}^K T_k \bigg( R^{\ts,1}_{t_k} + \gamma  \mathbb{E} \Big[ \big \|  H(\wt\theta_{t_k}) \big \|_{V_{t_k}^{-1}}|\mathcal{F}_{t_k}^x \Big]\bigg)\I_{E_{t_k}}.
%&+ \gamma \sum_{k=1}^K T_k  \mathbb{E} \Big[ \big \|  H(\wt\theta_{t_k}) \big \|_{V_{t_k}^{-1}}|\mathcal{F}_{t_k}^x \Big] \I_{E_{t_k}}.
\end{aligned}
\end{equation*}
We remove the expectation by adding and subtracting the actual realizations of $\wt\theta_{t_k}$ as
\begin{equation*}
R^{\ts,3}_{t_k} =  \mathbb{E} \Big[ \big \|  H(\wt\theta_{t_k}) \big \|_{V_{t_k}^{-1}}|\mathcal{F}_{t_k}^x \Big] -  \big \|  H(\wt\theta_{t_k}) \big \|_{V_{t_k}^{-1}}.
\end{equation*}
Thus, one obtains
\begin{equation*}
\begin{aligned}
R^{\ts} \leq& \sum_{k=1}^K T_k \Big(R^{\ts,1}_{t_k} \!+\!R^{\ts,3}_{t_k} \!+\! \gamma \big \|  H(\wt\theta_{t_k}) \big \|_{V_{t_k}^{-1}}\Big) \I_{E_{t_k} }.
%&+ \gamma \sum_{k=1}^K T_k     \I_{E_{t_k}}.
\end{aligned}
\end{equation*}
Now we want to relate the cumulative sum of the last regret term to $\sum_{t=1}^T \|z_t\|_{V^{-1}_t}$. This quantity represents the prediction error of the RLS, and we know from Prop.~\ref{eq:natural_explored_direction_ls} that it is bounded w.h.p. We now focus on the one-dimensional case, where $x_t$ is just a scalar value. Noticing that $\|z_t\|_{V_t^{-1}} = | x_t | \| H(\wt\theta_t)\|_{V_t^{-1}}$, one has:
\begin{equation*}
\sum_{t=0}^T \|z_t\|_{V^{-1}_t} = \sum_{k=1}^K \Big( \sum_{t = t_k}^{t_{k+1}-1} |x_t | \Big) \| H(\wt\theta_{t_k})\|_{V_{t}^{-1}}.
\end{equation*}
Intuitively, it means that over each episode, the more states are excited (e.g., the larger $\sum_{t = t_k}^{t_{k+1}-1} |x_t |$), the more $V_{t}^{-1}$ reduces in the direction $H(\wt\theta_{t_k})$. As a result, to ensure that the term $\sum_{k=1}^K T_k \| H(\wt\theta_{t_k})\|_{V_{t}^{-1}}$ in $R^{\ts}$ is small, it would be sufficient ti show that $\sum_{t = t_k}^{t_{k+1}-1} |x_t | \sim T_k$, i.e., that the states provides enough information to learn the system in each chosen direction $H(\wt\theta_{t_k})$.
More formally, let assume that there exists a constant $\alpha$ such that $T_k \leq \alpha \sum_{t = t_k}^{t_{k+1}-1} |x_t |$ for all $k \leq K$. Then, 
\begin{equation*}
\sum_{k=1}^K T_k \| H(\wt\theta_{t_k})\|_{V_{t_k}^{-1}} \leq \alpha \sum_{t=0}^T \| z_t\|_{V_{t_k}^{-1}} \leq 2\alpha \sum_{t=0}^T \| z_t\|_{V_{t}^{-1}},
\end{equation*}
where we use that $\det(V_{t}) \leq 2\det(V_{t_k})$ as guaranteed by the termination condition.
Unfortunately, the intrinsic randomness of $x_t$ (triggered by the noise $\xi_t$) is such that the assumption above is violated w.p. 1. However, in the one-dimensional case, the regret over the episode $k$ can be conveniently written as
\begin{equation*}
R_k(T) =  \Big( \sum_{t = t_k}^{t_{k+1}-1} |x_t |^2 \Big) \big( Q + K(\theta_{t_k})^2 R \big) - T_k J(\theta_*).
\end{equation*}
As a result, if we set
\begin{equation}
\alpha := X \frac{Q+RC^2}{J(\theta_*)} \geq X \frac{Q+R K(\theta_{t_k})^2}{J(\theta_*)}, 
\label{eq:alpha.definition}
\end{equation}
whenever $\sum_{t=t_k}^{t_{k+1}-1} \| x_t\| \leq \frac{1}{\alpha} T_k$ then we can directly conclude that $R_k(T)$ is zero. On the other hand, in the opposite case, we have $T_k \leq \alpha \sum_{t = t_k}^{t_{k+1}-1} |x_t |$ and thus we can upper bound the last term in $R^\ts$ as
\begin{equation*}
R^{\ts} \leq \sum_{k=1}^K T_k \Big(R^{\ts,1}_{t_k} \!+\!R^{\ts,3}_{t_k}\Big) \I_{E_{t_k} } + 2 \gamma \alpha \sum_{t=0}^T \|  z_t\|_{V_t^{-1}}.
\end{equation*}

\subsection{Final bound}\label{subsec:final.bound}

\textbf{Bounding $R^\rls_1$ and $R^\rls_3$.} These two terms can be bounded following similar steps as in~\citep{abbasi2011regret}. We report the detailed derivation in the supplement while here we simply report the final bounds
\begin{equation*}
R^\rls_1 \leq  \underbrace{2 D X^2 \sqrt{2\log(4/\delta)}}_{:=\gamma_1}\sqrt{T},
\end{equation*}
and
\begin{equation*}
\begin{aligned}
&R^\rls_3 %= \sum_{t=0}^T \big\{ z_t^\transp \wt\theta_t P(\wt\theta_t) \wt\theta_t^\transp z_t  - z_t^\top \theta_* P(\wt\theta_t) \theta_*^\transp z_t  \big\}\I_{E_t}, \\
\leq \underbrace{4 S D \sqrt{ (1 + C^2) X^2 } \mu_T(\delta^\prime)}_{:=\gamma_3} \sum_{t=0}^T \|z_t \|_{V_t^{-1}}\I_{E_t},
\end{aligned}
%\label{eq:rls.3.bound}
\end{equation*}
where $\mu_T(\delta^\prime) = \beta_T(\delta^\prime) + \gamma_T(\delta^\prime)$.
%
%\textbf{Bounding $R^\rls_1$.} By construction, $R^\rls_1$ is a martingale. Moreover, on $E$, thanks to the rejection sampling procedure, the increments are bounded almost surely and one can apply Azuma's inequility. Hence, with probability at least $1 - \delta/2$, 
%\begin{equation*}
%R^\rls_1 \leq  2 D X^2 \sqrt{2 T \log(4/\delta)}
%\end{equation*}
%We postpone the details in the supplement.\\

\textbf{Bounding $R^{\rls}_2$.}
Since the policy is updated from time to time, the difference of the optimal values $P(\wt\theta_t) - P(\wt\theta_{t+1})$ is zero unless when the parameters are updated. When it is the case, thanks to the rejection sampling procedure which ensures that every parameters belong to the set $\mathcal{S}$ of Asm.~\ref{asm:control.asm}, it is trivially bounded by $2 D$. Therefore, on event $E$, one has:
\begin{equation*}
\begin{aligned}
R^\rls_2 \leq 2 X^2 D K, %&= \sum_{t=1}^T \mathbb{E} [ x_{t+1}^\transp \big( P(\wt\theta_t) - P(\wt\theta_{t+1}) \big) x_{t+1} | \mathcal{F}_t ] \\
%&\leq 2 X^2 D K 
\end{aligned}
\end{equation*}
where $K$ is the (random) number of episodes.
By definition of \ts, the updates are triggered either when the $\det(V_t)$ increases by a factor $2$ or when the length of the episode is greater than $\tau$. Hence, the number of update can be split into $K = K^{det} + K^{len}$, where $K^{det}$ and $K^{len}$ are the number of updates triggered by the two conditions respectively. From Cor.~\ref{co:update.number.bound}, one gets:
\begin{equation*}
\begin{aligned}
K \leq \big(T/\tau + (n+d) \log_2 ( 1 + T X^2 (1+C^2)/\lambda) \big),
\end{aligned}
\end{equation*}
and thus
\begin{equation*}
\begin{aligned}
R^\rls_2 \leq \underbrace{2 X^2 D (n+d) \log_2 ( 1 + T X^2 (1+C^2)/\lambda)}_{:=\gamma_2} T/\tau.
\end{aligned}
\end{equation*}
%%

%\textbf{Bounding $R^\rls_3$.}
%This bound is very similar to the one of~\citet{abbasi2011regret} and is derived the same way. In the interest of the space we directly state the bound here and postpone the derivation of the proof in the supplement, for sake of completeness.
%\begin{equation}
%\begin{aligned}
%&R^\rls_3 = \sum_{t=0}^T \big\{ z_t^\transp \wt\theta_t P(\wt\theta_t) \wt\theta_t^\transp z_t  - z_t^\top \theta_* P(\wt\theta_t) \theta_*^\transp z_t  \big\}\I\{E_t\}, \\
%&\leq \underbrace{2 S D \sqrt{ (1 + C^2) X^2 } \big( \beta_T(\delta^\prime) + \gamma_T(\delta^\prime) \big)}_{:=\gamma_3} \sum_{t=0}^T \|z_t \|_{V_t^{-1}}\I\{E_t\} .
%\end{aligned}
%\label{eq:rls.3.bound}
%\end{equation}

\textbf{Plugging everything together.}
We are now ready to bring all the regret terms together and obtain
\begin{equation*}
\begin{aligned}
R(T) & \leq (2 \gamma\alpha  + \gamma_3 )\sum_{t=0}^T \|z_t\|_{V^{-1}_{t}} \I_{E_{t}} + \gamma_2 T/\tau  \\
&\quad+ \gamma_1\sqrt{T} +\sum_{k=1}^K T_k \big(R^{\ts,1}_{t_k} + R^{\ts,3}_{t_k}\big) \I_{E_{t_k} }
\end{aligned}
\end{equation*}

 At this point, the regret bound is decomposed into several parts: 1) the first term can be bounded as $\sum_{t=0}^T \|z_t \|_{V_t^{-1}} = \tilde{O} (\sqrt{T})$ on $E$ using Prop.~\ref{p:self_normalized_determinant} (see App.~\ref{sec:app_regret_proofs} for details) 2) two terms which are already conveniently bounded as $T/\tau$ and $\sqrt{T}$, and 3) two remaining terms from $R^\ts$ that are almost exact martingales. In fact, $T_k$ is random w.r.t.\ $\mathcal{F}_{t_k}$ and thus the terms $T_k R_{t_k}^{\ts,1}$ and $T_k R_{t_k}^{\ts,3}$ are not proper martingale difference sequences. However, we can leverage on the fact that on most of the episodes, the length $T_k$ is not random since the termination of the episode is triggered by the (deterministic) condition $T_k \leq \tau$.
 Let $\alpha_k = (R^{\ts,1}_{t_k} + R^{\ts,3}_{t_k} )\I_{E_{t_k}}$, $\mathcal{K}^{\text{det}}$ and $\mathcal{K}^{\text{len}}$ two set of indexes of cardinality $K^{\text{det}}$ and $K^{\text{len}}$ respectively, which correspond to the episodes terminated following the determinant or the limit condition respectively. Then, we can write
 \begin{equation*}
 \begin{aligned}
& \sum_{k=1}^K T_k \alpha_k = %\sum_{k \in \mathcal{K}^{\text{det}}} T_k \alpha_k  + \sum_{k \in \mathcal{K}^{\text{len}}} T_k \alpha_k\\
  \sum_{k \in \mathcal{K}^{\text{det}}} T_k \alpha_k + \tau  \sum_{k \in \mathcal{K}^{\text{len}}} \alpha_k\\
 &\leq  \sum_{k \in \mathcal{K}^{\text{det}}} T_k \alpha_k +\!\! \sum_{k \in \mathcal{K}^{\text{len}}} \tau\alpha_k +\!\! \sum_{k \in \mathcal{K}^{\text{det}}} \tau\alpha_k +\!\! \sum_{k \in \mathcal{K}^{\text{det}}} \tau\|\alpha_k \| \\
% &\quad+ \tau  \sum_{k \in \mathcal{K}^{det}} \alpha_k + \tau \sum_{k \in \mathcal{K}^{det}} \|\alpha_k \| \\
 &\leq 2 \tau  \sum_{k \in \mathcal{K}^{\text{det}}}  \| \alpha_k \| + \tau \sum_{k=1}^K \alpha_k.
 \end{aligned}
 \end{equation*}
The first term can be bounded using Lem.~\ref{le:det.update.number.bound}, which implies that the number of episodes triggered by the determinant condition is only logarithmic. On the other hand the remaining term $\sum_{k=1}^K \alpha_k$ is now a proper martingale and, together with the boundedness of $\alpha_k$ on event $E$, Azuma inequality directly holds. We obtain
\begin{equation*}
\sum_{k=1}^K T_k \big(R^{\ts,1}_{t_k} + R^{\ts,3}_{t_k}\big) \I_{E_{t_k} } = \tilde{O}(\tau \sqrt{K}).
\end{equation*}
w.p.\ $1 - \delta/2$. Grouping all higher-order terms w.r.t. to $T$ and applying Cor.~\ref{co:update.number.bound} to bound $K$, we finally have
\begin{equation*}
R(T) \leq C_1 \frac{T}{\tau} + C_2 \tau \sqrt{T/\tau},
\end{equation*}
where $C_1$ and $C_2$ are suitable problem-dependent constants. This final bound is optimized for $\tau = O(T^{1/3})$ and it induces the final regret bound $R(T) = O(T^{2/3})$. More details are reported in App.~\ref{sec:app_regret_proofs}.

% !TEX root = main.tex

%%%%%%%%%%%%%%%%%%%%%%%%%%%%%%%%%%%%%%%%%%%%%%%%%%%%%%%%%%%%%%%%%%%%%%
%%%%%%%%%%%%%%%%%%%%%%%%%%%%%%%%%%%%%%%%%%%%%%%%%%%%%%%%%%%%%%%%%%%%%%
%%%%%%%%%%%%%%%%%%%%%%%%%%%%%%%%%%%%%%%%%%%%%%%%%%%%%%%%%%%%%%%%%%%%%%

\section{Discussion}\label{sec:discussion}

We derived the first frequentist regret for \ts in LQ control systems. Despite the existing results in LQ for optimistic approaches (\ofulq), the Bayesian analysis of \ts in LQ, and its frequentist analysis in linear bandit, we showed that controlling the frequentist regret induced by the randomness of the sampling process in LQ systems is considerably more difficult and it requires developing a new line of proof that directly relates the regret of \ts and the controls executed over time. Furthermore, we show that \ts has to solve a trade-off between frequently updating the policy to guarantee enough optimistic samples and reducing the number of policy switches to limit the regret incurred at each change. This gives rise to a final bound of $O(T^{2/3})$. This opens a number of questions. \textbf{1)} The current analysis is derived in the general $n/d$-dimensional case except for Lem.~\ref{le:probability_optimistic} and the steps leading to the introduction of the state in Sect.~\ref{subsec:final.bound}, where we set $n=d=1$. We believe that these steps can be extended to the general case without affecting the final result. \textbf{2)} The final regret bound is in striking contrast with previous results for \ts. While we provide a rather intuitive reason on the source of this extra regret, it is an open question whether a different \ts or analysis could allow to improve the regret to $O(\sqrt{T})$ or whether this result reveals an intrinsic limitation of the randomized approach of \ts.

%\begin{itemize}
%\item 1D
%\item $T^{2/3}$
%\item $J(\theta_\star)$
%\end{itemize}

{\small
\textbf{Acknowledgement} This research is supported in part by a grant from CPER Nord-Pas de Calais/FEDER DATA Advanced data science and technologies 2015-2020, CRIStAL (Centre de Recherche en Informatique et Automatique de Lille), and the French National Research Agency (ANR) under project ExTra-Learn n.ANR-14-CE24-0010-01.
}

\newpage
\begin{small}
\bibliography{biblio}

\begin{thebibliography}{14}
\providecommand{\natexlab}[1]{#1}
\providecommand{\url}[1]{\texttt{#1}}
\expandafter\ifx\csname urlstyle\endcsname\relax
  \providecommand{\doi}[1]{doi: #1}\else
  \providecommand{\doi}{doi: \begingroup \urlstyle{rm}\Url}\fi

\bibitem[Abbasi-Yadkori and Szepesv{\'a}ri(2011)]{abbasi2011regret}
Yasin Abbasi-Yadkori and Csaba Szepesv{\'a}ri.
\newblock Regret bounds for the adaptive control of linear quadratic systems.
\newblock In \emph{COLT}, pages 1--26, 2011.

\bibitem[Abbasi-Yadkori and Szepesv{\'a}ri(2015)]{abbasi2015bayesian}
Yasin Abbasi-Yadkori and Csaba Szepesv{\'a}ri.
\newblock Bayesian optimal control of smoothly parameterized systems.
\newblock In \emph{Proceedings of the Conference on Uncertainty in Artificial
  Intelligence}, 2015.

\bibitem[Abbasi-Yadkori et~al.(2011)Abbasi-Yadkori, P{\'a}l, and
  Szepesv{\'a}ri]{abbasi-yadkori2011improved}
Yasin Abbasi-Yadkori, D{\'a}vid P{\'a}l, and Csaba Szepesv{\'a}ri.
\newblock Improved algorithms for linear stochastic bandits.
\newblock In \emph{Proceedings of the 25th Annual Conference on Neural
  Information Processing Systems (NIPS)}, 2011.

\bibitem[Acosta and Dur{\'a}n(2004)]{acosta2004optimal}
Gabriel Acosta and Ricardo~G Dur{\'a}n.
\newblock An optimal poincar{\'e} inequality in l 1 for convex domains.
\newblock \emph{Proceedings of the american mathematical society}, pages
  195--202, 2004.

\bibitem[Agrawal and Goyal(2012)]{agrawal2012thompson}
Shipra Agrawal and Navin Goyal.
\newblock Thompson sampling for contextual bandits with linear payoffs.
\newblock \emph{arXiv preprint arXiv:1209.3352}, 2012.

\bibitem[Bittanti and Campi(2006)]{bittanti2006adaptive}
S~Bittanti and MC~Campi.
\newblock Adaptive control of linear time invariant systems: the ``bet on the
  best'' principle.
\newblock \emph{Communications in Information \& Systems}, 6\penalty0
  (4):\penalty0 299--320, 2006.

\bibitem[Campi and Kumar(1998)]{campi1998adaptive}
Marco~C Campi and PR~Kumar.
\newblock Adaptive linear quadratic gaussian control: the cost-biased approach
  revisited.
\newblock \emph{SIAM Journal on Control and Optimization}, 36\penalty0
  (6):\penalty0 1890--1907, 1998.

\bibitem[Jaksch et~al.(2010)Jaksch, Ortner, and Auer]{jaksch2010near-optimal}
Thomas Jaksch, Ronald Ortner, and Peter Auer.
\newblock Near-optimal regret bounds for reinforcement learning.
\newblock \emph{J. Mach. Learn. Res.}, 11:\penalty0 1563--1600, August 2010.

\bibitem[Lancaster and Rodman(1995)]{lancaster1995algebraic}
P.~Lancaster and L.~Rodman.
\newblock \emph{Algebraic riccati equations}.
\newblock Oxford University Press, 1995.

\bibitem[Osband and {Van Roy}(2016)]{osband2015bootstrapped}
Ian Osband and Benjamin {Van Roy}.
\newblock Deep exploration via bootstrapped dqn.
\newblock \emph{Proceedings of the 30th Annual Conference on Neural Information
  Processing Systems (NIPS)}, 2016.

\bibitem[Osband et~al.(2016)Osband, {Van Roy}, and
  Wen]{osband2016generalization}
Ian Osband, Benjamin {Van Roy}, and Zheng Wen.
\newblock Generalization and exploration via randomized value functions.
\newblock In \emph{Proceedings of the 33nd International Conference on Machine
  Learning, {ICML} 2016, New York City, NY, USA, June 19-24, 2016}, pages
  2377--2386, 2016.

\bibitem[Payne and Weinberger(1960)]{payne1960optimal}
Lawrence~E Payne and Hans~F Weinberger.
\newblock An optimal poincar{\'e} inequality for convex domains.
\newblock \emph{Archive for Rational Mechanics and Analysis}, 5\penalty0
  (1):\penalty0 286--292, 1960.

\bibitem[Strehl et~al.(2009)Strehl, Li, and Littman]{strehl2009reinforcement}
Alexander~L. Strehl, Lihong Li, and Michael~L. Littman.
\newblock Reinforcement learning in finite mdps: Pac analysis.
\newblock \emph{J. Mach. Learn. Res.}, 10:\penalty0 2413--2444, December 2009.

\bibitem[Strens(2000)]{strens2000a-bayesian}
Malcolm J.~A. Strens.
\newblock A bayesian framework for reinforcement learning.
\newblock In \emph{Proceedings of the Seventeenth International Conference on
  Machine Learning}, pages 943--950, 2000.

\end{thebibliography}
\bibliographystyle{plainnat}
\end{small}

\newpage
\onecolumn
\appendix
% !TEX root = main.tex

\section{Control theory}\label{sec:app_control_theory}
\subsection{Proof of Prop.~\ref{prop:admissib_param_set}}
\begin{enumerate}
\item When $\theta^\top = (A,B)$ is not stabilizable, there exists no linear control $K$ such that the controlled process $x_{t+1}  = A x_{t} + B K x_{t} + \epsilon_{t+1}$ is stationary. Thus, the positiveness of $Q$ and $R$ implies $J(\theta) = \Tr (P(\theta)) = + \infty$. As a consequence, $\theta^\top \notin \mathcal{S}$.
\item The mapping $\theta \rightarrow \Tr(P(\theta))$ is continuous (see Lem.~\ref{le:riccati.regularity}). Thus, $\mathcal{S}$ is compact as the intersection between a closed and a compact set.
\item The continuity of the mapping $\theta \rightarrow K(\theta)$ together with the compactness of $\mathcal{S}$ justifies the finite positive constants $\rho$ and $C$. Moreover, since every $\theta \in \mathcal{S}$ are stabilizable pairs, $\rho < 1$.
\end{enumerate}
\subsection{Proof of Lem.~\ref{le:riccati.regularity} }
Let $\theta^\transp = (A,B)$ where $A$ and $B$ are matrices of size $n \times n$ and $n \times d$ respectively. Let $\mathcal{R} : \mathbb{R}^{n+d,n} \times \mathbb{R}^{n,n} \rightarrow  \mathbb{R}^{n,n} $ be the Riccati operator defined by:
\begin{equation}\label{eq:riccati.operator}
\mathcal{R} (\theta, P) := Q - P + A^\transp P A - A^\transp P B (R + B^\transp P B)^{-1} B^\transp P A,
\end{equation}
where $Q,R$ are positive definite matrices. Then, the solution $P(\theta)$ of the Riccati equation of Thm.~\ref{th:lqr} is the solution of $\mathcal{R}(\theta,P) = 0$. While Prop.~\ref{prop:admissib_param_set} guarantees that there exists a unique admissible solution as soon as $\theta \in \mathcal{S}$, addressing the regularity of the function $\theta \rightarrow P(\theta)$ requires the use of the implicit function theorem.
\begin{theorem}[Implicit function theorem]\label{th:implicit.function.theorem}
Let $E$ and $F$ be two banach spaces, let $\Omega \subset E \times F$ be an open subset. Let $f : \Omega \rightarrow F$ be a $C^1$-map and let $(x_0,y_0)$ be a point of $\Omega$
 such that $f(x_0,y_0) = 0$. We denote as $d_y f(x_0,y_0) : F \rightarrow F$ the differential of the function $f$ with respect to the second argument at point $(x_0,y_0)$. Assume that this linear transformation is bounded and invertible. Then, there exists 
 \begin{enumerate}
\item two open subsets $U$ and $V$ such that $(x_0,y_0) \in U \times V \subset \Omega$, 
\item a function $g : U \rightarrow V$ such that $g(x) = y$ for all $(x,y) \in U \times V$.
\end{enumerate}
Moreover, $g$ is $C^1$ and $d g(x)  = - d_y f(x,g(x))^{-1} d_x f(x,g(x))$ for all $(x,y) \in U \times V$.
\end{theorem}
Since $R$ is positive definite, the Riccati operator is clearly a $C^1$-map. Moreover, thanks to Thm.~\ref{th:lqr}, to any $\theta \in \mathcal{S}$, there exists an admissible $P$ such that $\mathcal{R}(\theta,P) = 0$. Thanks to Thm.~\ref{th:implicit.function.theorem}, a sufficient condition for $\theta \rightarrow P(\theta)$ to be $C^1$ on $\mathcal{S}$ is that the linear map
$d_P\mathcal{R}(\theta,P(\theta)) : \mathbb{R}^{n\times n}\rightarrow \mathbb{R}^{n\times n}$ is a bounded invertible transformation i.e. 
\begin{itemize}
\item \textbf{Bounded.} There exists $M$ such that, for any $P \in \mathbb{R}^{n \times n }$, $ \| d_P\mathcal{R}(\theta,P(\theta)) ( P) \| \leq M \| P\|$.
\item \textbf{Invertible.} There exists a bounded linear operator $S : \mathbb{R}^{n \times n } \rightarrow {R}^{n \times n }$ such that $S P  =I_{n,n} $ and $P S = I_{n,n}$.
\end{itemize}

\begin{lemma}\label{le:riccati.gradient.lyapunov}
Let $\theta^\transp = (A,B)$ and $\mathcal{R}$ be the Riccati operator defined in equation~\eqref{eq:riccati.operator}. Then, the differential of $\mathcal{R}$ w.r.t $P$ taken in $(\theta, P(\theta))$ denoted as $d_P\mathcal{R}(\theta,P(\theta))$ is defined by:
\begin{equation*}
d_P\mathcal{R}(\theta,P(\theta))( \delta P) := A_c^T \delta P A_c - \delta P, \quad \text{for any } \delta P \in \mathbb{R}^{n \times n},
\end{equation*}
where $A_c = A - B (R + B^\transp P B)^{-1} B^\transp P(\theta) A$.
\end{lemma}
\begin{proof}
The proof is straightforward using the standard composition/multiplication/inverse operations for the differential operator together with an appropriate rearranging.
\end{proof}
Clearly, $d_P\mathcal{R}(\theta,P(\theta))$ is a bounded linear map. Moreover, thanks to the Lyapunov theory, for any stable matrix $\| A_c \|_2 < 1$ and for any matrix $Q$, the Lyapunov equation $A_c^T X A_c - X = Q$ admits a unique solution. From Thm.~\ref{th:lqr}, the optimal matrix $P(\theta)$ is such that the corresponding $A_c$ is stable. This implies that  $d_P\mathcal{R}(\theta,P(\theta))$ is an invertible operator, and $\theta \rightarrow P(\theta)$ is $C^1$ on $\mathcal{S}$.\\

Therefore, the differential of $\theta \rightarrow P(\theta)$ can be deduced from the implicit function theorem. After tedious yet standard operations, one gets that for any $\theta \in \mathcal{S}$ and direction $\delta \theta \in \mathbb{R}^{(n+d)\times n}$: 
$$d J(\theta)(\delta \theta) = \Tr (d P(\theta)(\delta \theta)) = \Tr( \nabla J(\theta)^\transp \delta \theta ),$$
 where $\nabla J(\theta) \in \mathbb{R}^{(n+d)\times n}$ is the jacobian matrix of $J$ in $\theta$. For any $\delta \theta \in \mathbb{R}^{(n+d)\times n}$, one has: 
 \begin{equation}
\nabla J(\theta)^\transp \delta \theta = A_c(\theta)^\transp \nabla J(\theta)^\transp \delta \theta  A_c(\theta)  + C(\theta,\delta\theta) +  C(\theta,\delta\theta)^\transp, \quad \text{where} \quad C(\theta,\delta\theta) = A_c(\theta)^\transp P(\theta) \delta \theta^\transp H(\theta).
\label{eq:gradient.J.equation}
\end{equation}
\begin{proposition}\label{p:gradient.inequality}
For any $\theta \in \mathcal{S}$ and any positive definite matrix $V$, one has the following inequality for the weighted norm of the gradient of $J$:
\begin{equation*}
\| \nabla J(\theta) \|_V \leq \|A_c(\theta)\|_2^2 \| \nabla J(\theta) \|_V + 2  \|P(\theta)\| \|A_c(\theta)\|_2 \| H(\theta)\|_V.
\end{equation*}
\end{proposition}

\begin{proof}
 For any $\theta \in \mathcal{S}$ and any positive definite matrix $V \in \mathbb{R}^{(n+d)\times(n+d)}$ . Applying~\eqref{eq:gradient.J.equation} to $\delta \theta = V \nabla J(\theta)$ leads to:
\begin{equation*}
\nabla J(\theta)^\transp V \nabla J(\theta) = A_c(\theta)^\transp \nabla J(\theta)^\transp V \nabla J(\theta) A_c(\theta) + C(\theta,V \nabla J(\theta)) + C(\theta,V \nabla J(\theta))^\transp,
\end{equation*}
where $C(\theta,V \nabla J(\theta))^\transp = \big( V^{1/2} H(\theta) \big)^\transp V^{1/2} \nabla J(\theta) P(\theta) A_c(\theta)$. Let $\langle A , B \rangle = \Tr A^\transp B$ be the Frobenius inner product, then taking the trace of the above equality, one gets:
\begin{equation*}
\| \nabla J(\theta) \|^2_V = \|\nabla J(\theta) A_c(\theta) \|^2_V + 2 \big \langle V^{1/2} H(\theta) , V^{1/2} \nabla J(\theta) P(\theta) A_c(\theta) \big \rangle.
\end{equation*}
Using the Cauchy-Schwarz inequality and that the Frobenius norm is sub-multiplicative together with $\Tr(M_1M_2) \leq \|M_1\|_2 \Tr(M_2)$ for any $M_1,M_2$ symmetric positive definite matrices, one obtains:
\begin{equation*}
\| \nabla J(\theta) \|^2_V \leq \|A_c(\theta)\|_2^2 \| \nabla J(\theta) \|^2_V + 2 \| H(\theta)\|_V \|P(\theta)\| \|A_c(\theta)\|_2 \| \nabla J(\theta)\|_V.
\end{equation*}
Finally, dividing by $\| \nabla J(\theta)\|_V$ provides the desired result.
\end{proof}

% !TEX root = main.tex

\section{Material}\label{sec:app_ofu_material}

\begin{theorem}[Azuma's inequality]\label{th:azuma.ineq}
Let $ \{ M_s \}_{s \geq 0} $ be a super-martingale such that $ | M_s - M_{s-1} | \leq c_s$ almost surely. Then, for all $t > 0$ and all $\epsilon  > 0$, 
\begin{equation*}
\mathbb{P} \big( | M_t - M_0| \geq \epsilon \big) \leq 2 \exp \Big( \frac{ - \epsilon^2}{ 2 \sum_{s=1}^t c_s^2} \Big).
\end{equation*}
\end{theorem}

\begin{lemma}[Lemma.~8 from~\citet{abbasi2011regret}]\label{le:det.update.number.bound}
Let $K^{det}$ be the number of changes in the policy of Algorithm~\ref{alg:ts} due to the determinant trigger $\det(V_t) \geq 2 \det(V_0)$. Then, on $E$, $K^{det}$ is at most
\begin{equation*}
K^{den} \leq  (n+d) \log_2 ( 1 + T X^2 (1+C^2)/ \lambda).
\end{equation*}
\end{lemma}

\begin{corollary}\label{co:update.number.bound}
Let $K$ be the number of policy changes of Algorithm~\ref{alg:ts}, $K^{det}$ be defined as in Lem.~\ref{le:det.update.number.bound} and $K^{len}  = K - K^{det}$ be the number of policy changes due to the length trigger $t \geq t_0 + \tau$. Then, on $E$, $K$ is at most
\begin{equation*}
K \leq K^{det} + K^{len} \leq (n+d) \log_2 ( 1 + T X^2 (1+C^2)/ \lambda) + T/\tau.
\end{equation*}
Moreover, assuming that $T \geq \frac{\lambda}{X^2 (1+C^2)}$, one gets $K \leq (n+d) \log_2 ( 1 + T X^2 (1+C^2)/ \lambda) T/\tau$.
\end{corollary}

\begin{lemma}[Chernoff bound for Gaussian r.v.]\label{le:chernoff.bound}

Let $X\sim\mathcal{N}(0,1)$. For any $0 < \delta <1$, for any $t\geq 0$, then,
\begin{equation*}
\mathbb{P}(|X| \geq t ) \leq 2 \exp \big(-\frac{t^2}{2} \big).
\end{equation*}
\end{lemma}
\textbf{Proof of Lem.\ref{le:high_proba_concentration}.}
Let $\delta^\prime = \delta/8T$.
\begin{enumerate}
\item From Prop.~\ref{p:concentration}, $\mathbb{P} \big( \| \wh\theta_t - \theta_*\|_{V_t} \leq \beta_t(\delta^\prime) \big) \geq 1 - \delta^\prime$. Hence, 
\begin{equation*}
\begin{aligned}
\mathbb{P}\big( \wh{E} \big) &= \mathbb{P}\Big( \bigcap_{t =0}^T \big( \| \wh\theta_t - \theta_*\|_{V_t} \leq \beta_t(\delta^\prime) \big) \Big)\\
 &= 1  - \mathbb{P}\Big( \bigcup_{t =0}^T \big( \| \wh\theta_t - \theta_*\|_{V_t} \geq \beta_t(\delta^\prime) \big) \Big) \\
&\geq 1 - \sum_{t=0}^T  \mathbb{P} \big( \| \wh\theta_t - \theta_*\|_{V_t} \geq \beta_t(\delta^\prime) \big)  \\
&\geq 1 - T\delta^\prime \geq 1 - \delta/8
\end{aligned}
\end{equation*}
\item From Lem.~\ref{le:chernoff.bound}, let $\eta \sim \distro$ then, for any $\epsilon > 0$, making use of the fact that $\|\eta\| \leq n \sqrt{n+d} \max_{i\leq n+d,j\leq n } |\eta_{i,j}|$, 
\begin{equation*}
\mathbb{P} \big( \|\eta \| \leq \epsilon \big) \geq \mathbb{P} \big( n \sqrt{n+d} \max_{i,j} |\eta_{i,j}| \leq \epsilon \big) 
\geq 1 - \prod_{i,j} \mathbb{P} \big( |\eta_{i,j} | \geq \frac{\epsilon}{n \sqrt{n+d}} \big) \geq 1 - n(n+d) \mathbb{P}_{X \sim \mathcal{N}(0,1)} \big( |X| \geq \frac{\epsilon}{n \sqrt{n+d}} \big).
\end{equation*}
Hence,
\begin{equation*}
\begin{aligned}
\mathbb{P}\big( \wt{E} \big) &= \mathbb{P}\Big( \bigcap_{t =0}^T \big( \| \wt\theta_t - \wh\theta_t\|_{V_t} \leq \gamma_t(\delta^\prime) \big) \Big) 
= 1  - \mathbb{P}\Big( \bigcup_{t =0}^T \big( \| \wt\theta_t - \wh\theta_t\|_{V_t} \geq \gamma_t(\delta^\prime) \big) \Big) \\
  &\geq 1 - \sum_{t=0}^T  \mathbb{P} \big( \| \wt\theta_t - \wh\theta_t\|_{V_t} \geq \gamma_t(\delta^\prime) \big)  
  \geq 1 - \sum_{t=0}^T  \mathbb{P} \big( \| \eta \| \geq \gamma_t(\delta^\prime)/\beta_t(\delta^\prime) \big)  \\
    &\geq 1 - \sum_{t=0}^T  \mathbb{P} \Big( \| \eta \| \geq n \sqrt{2 (n+d) \log\big( 2 n (n+d) / \delta^\prime}\big) \Big)  \\
  &\geq 1 - T\delta^\prime \geq 1 - \delta/8.
\end{aligned}
\end{equation*}
\item Finally, a union bound argument ensures that $\mathbb{P}(\wh{E} \cap \wt{E}) \geq 1 - \delta/4$.
\end{enumerate}

\textbf{Proof of Cor.~\ref{co:high_proba_boundedness}.} This result comes directly from Sec.~4.1. and App.~D of~\citet{abbasi2011regret}. The proof relies on the fact that, on $\wh{E}$, because $\wt\theta_t$ is chosen within the confidence ellipsoid $\calE^\rls_t$, the number of time steps the true closed loop matrix $A_* + B_* K(\wt\theta_t)$ is unstable is small. Intuitively, the reason is that as soon as the true closed loop matrix is unstable, the state process explodes and the confidence ellipsoid is drastically changed. As the ellipsoid can only shrink over time, the state is well controlled expect for a small number of time steps.\\
Since the only difference is that, on $\wh{E} \cap \wt{E}$,  $\wt\theta_t \in \calE^\ts_t$, the same argument applies and the same bound holds replacing $\beta_t$ with $\gamma_t$. Therefore, there exists appropriate problem dependent constants $X,X^\prime$ such that $ \mathbb{P}( \bar{E} | \wh{E} \cap \wt{E} ) \geq 1 - \delta/4$. Finally, a union bound argument ensures that
$\mathbb{P}(\wh{E} \cap \wt{E} \cap \bar{E}) \geq 1 - \delta/2$.

% !TEX root = main.tex

\section{Proof of Lem.~\ref{le:probability_optimistic}}\label{sec:app_optimism}
We prove here that, on $E$, the sampling $\wt\theta \sim \mathcal{R}_\mathcal{S} ( \wh \theta_t + \beta_t(\delta^\prime) V^{1/2}_{t} )$ guarantees a fixed probability of sampling an optimistic parameter, i.e. which belongs to $\Theta_t^{\opt} := \{ \theta \in\Re^d \hspace{1mm} | \hspace{1mm} J(\theta) \leq J(\theta^\star) \}$. However, our result only holds for the $1-$dimensional case as we deeply leverage on the geometry of the problem. Figure~\ref{fig:optimistic_set} synthesizes the properties of the optimal value function and the geometry of the problem w.r.t the probability of being optimistic.
\begin{figure}[h]
\begin{center}
   \includegraphics[scale=0.5]{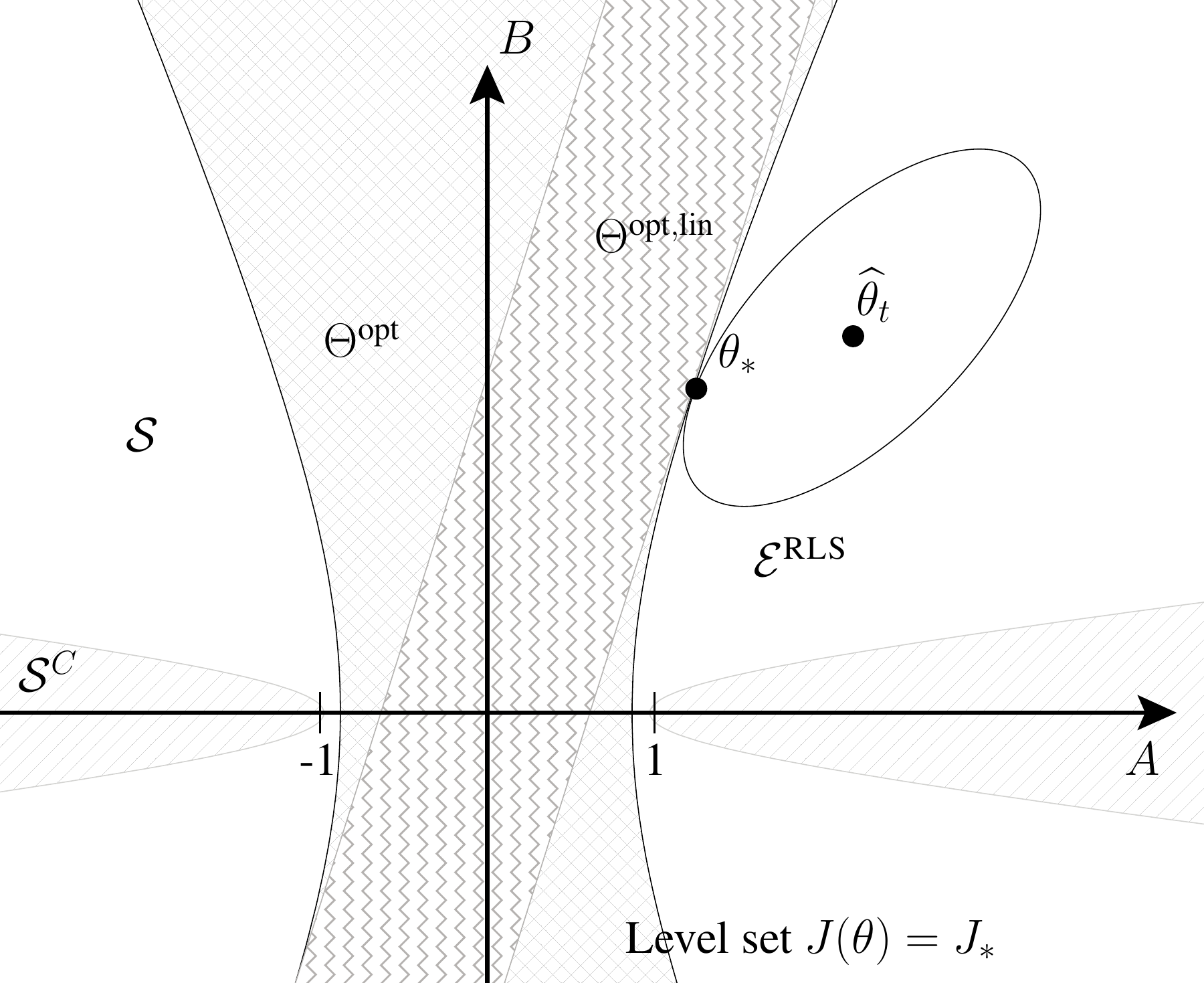}
   \caption{\label{fig:optimistic_set} \textbf{Optimism and worst case configuration.} \textbf{1)} In 1-D, the Riccati solution is well-defined expect for $\{ (A,B) \in ]-\infty,-1]\cup[1,\infty[ \times \{0\} \}$. The rejection sampling procedure into $\mathcal{S}$ ensures $P(\wt\theta_t)$ to be well-defined. Moreover, $\mathcal{S}^c$ does not overlap with $\Theta^\opt$. \textbf{2)} The introduction of the subset $\Theta^{lin,\opt }$ prevents using the actual - yet complicated - optimistic set $\Theta^\opt$ to lower bound the probability of being optimistic. \textbf{3)} Even if the event $\calE^\rls$ holds, there exists an ellipsoid configuration which does not contain any optimistic point. This justifies the over-sampling to guarantee a fixed probability of being optimistic.}
   \end{center}
\end{figure}

\begin{enumerate}
\item First, we introduce a simpler subset of optimistic parameters which involves hyperplanes rather than complicated $J$ level sets. Without loss of generality we assume that 
$A_* + B_* K_* = \rho_* \geq 0$ and introduce $H_* = \begin{pmatrix} 1 \\ K_* \end{pmatrix} \in \mathbb{R}^2$ so that $A_* + B_* K_* = \theta^\transp H_*$. Let $\Theta^{lin,\opt} = \{ \theta \in\Re^d \hspace{1mm} | \hspace{1mm}  | \theta^\transp H_* | \leq \rho_* \}$. Intuitively, $\Theta^{lin,\opt} $ consists in the set of systems $\theta$ which are more stable under control $K_*$. The following proposition ensures those systems to be optimistic.

\begin{proposition}\label{p:stability.optimism}
$\Theta^{lin,\opt} \subset \Theta_t^{\opt}$.
\end{proposition}
\begin{proof}
Leveraging on the expression of $J$, one has when $n=d=1$, 
\begin{equation*}
J(\theta) = \Tr (P(\theta)) = P(\theta) = \lim_{T \rightarrow \infty} \sum_{t=0}^T x_t^2 (Q + K(\theta)^2 R) =  (Q + K(\theta)^2 R) \mathbb{V}(x_t),
\end{equation*}
where $\mathbb{V}(x_t) = (1 - |\theta^\transp H(\theta) |^2)^{-1}$ is the steady-state variance of the stationary first order autoregressive process $x_{t+1} = \theta^\transp H(\theta) x_t + \epsilon_{t+1}$ where $\epsilon_{t}$ is zero mean noise of variance $1$ and $H(\theta) = \begin{pmatrix} 1 \\ K(\theta) \end{pmatrix}$. Thus,
\begin{equation*}
J(\theta) = \big( Q + K(\theta)^2 R \big)  \big(1 - | \theta^\transp H(\theta) |^2 \big)^{-1}.
\end{equation*}
Hence, for any $\theta \in \Theta^{lin,\opt}$, $(1 - |\theta^\transp H_* |^2)^{-1} \leq (1 - | \theta_*^\transp H_* |^2)^{-1}$ which implies that 
\begin{equation*}
(Q + K_*^2 R ) (1 - |\theta^\transp H_* |^2)^{-1} \leq  (Q + K_*^2 R )  (1 - |\theta_*^\transp H_* |^2)^{-1} = J(\theta_*).
\end{equation*}
However, since $K(\theta)$ is the optimal control associated with $\theta$, 
\begin{equation*}
\begin{aligned}
J(\theta) &= (Q + K(\theta) ^2 R ) (1 - |\theta^\transp H(\theta) |^2)^{-1}\\
 &= \min_{K}  (Q + K^2 R ) (1 - |\begin{pmatrix} 1 & K \end{pmatrix} \theta |^2)^{-1} \\
 &\leq (Q + K_*^2 R ) (1 - |\theta^\transp H_* |^2)^{-1} \\
 &\leq J(\theta_*)
 \end{aligned}
\end{equation*}
\end{proof}
As a result,  $\mathbb{P}\big( \wt\theta_t \in \Theta^\opt \hspace{1mm} | \hspace{1mm} \F^x_{t}, \wh E_t \big) \geq \mathbb{P}\big(\wt\theta_t \in \ \Theta^{lin,\opt} \hspace{1mm} | \hspace{1mm} \F^x_{t}, \wh E_t \big)$ and we can focus on $\Theta^{lin,\opt}$.
\item To ensure the sampling parameter to be admissible, we perform a rejection sampling until $\wt\theta_t \in \mathcal{S}$. Noticing that $\Theta^{lin,\opt} \subset \Theta^{\opt} \subset \mathcal{S}$ by construction, the rejection sampling is always favorable in terms of probability of being optimistic. Since we seek for a lower bound, we can get rid of it and consider  
$\wt\theta_t = \wh\theta_t + \beta_t(\delta^\prime) V^{-1/2}_t \eta$ where $\eta \sim \mathcal{N}(0,I_2)$.\footnote{In the 1-dimensional case, $\eta$ is just a 2d standard gaussian r.v.}
\item On $\wh{E}_t$, $\theta_\star \in \calE^\rls_t$, where $\calE^\rls_t$ is the confidence RLS ellipsoid centered in $\wh\theta_t$. Since $\theta_*$ is fixed (by definition), we lower bound the probability by considering the worst possible $\wh\theta_t$ such that $\wh{E}_t$ holds. Intuitively, we consider the worst possible center for the RLS ellipsoid such that $\theta_\star$ still belong in $\calE^\rls_t$ and that the probability of being optimistic is minimal. Formally,
\begin{equation*}
\begin{aligned}
 \mathbb{P}\big(\wt\theta_t \in \Theta^{lin,\opt} \hspace{1mm} | \hspace{1mm} \F^x_{t}, \wh E_t \big) &= 
  \mathbb{P}_{\wt\theta_t  \sim \mathcal{N}(\wh\theta_t, \beta_t^2(\delta^\prime) V_t^{-1})} \big(\wt\theta_t \in \Theta^{lin,\opt} \hspace{1mm} | \hspace{1mm} \F^x_{t}, \wh E_t \big) \\
  &\geq \min_{\wh\theta : \| \wh\theta - \theta_*\|_{V_t} \leq \beta_t(\delta^\prime)}   \mathbb{P}_{\wt\theta_t  \sim \mathcal{N}(\wh\theta, \beta_t^2(\delta^\prime) V_t^{-1})} \big(\wt\theta_t  \in \Theta^{lin,\opt} \hspace{1mm} | \hspace{1mm} \F^x_{t}) 
\end{aligned}
\end{equation*}
Moreover, by Cauchy-Schwarz inequality, for any $\wh\theta$, 
\begin{equation*}
 \big| (\wh \theta - \theta_* )^\transp H_* \big| \leq \|\wh \theta - \theta_* \|_{V_t}  \|H_* \|_{V_{t}^{-1}} \leq \beta_t(\delta^\prime) \|H_*\|_{V_{t}^{-1}},
 \end{equation*}
thus,
\begin{equation}
\label{eq:optimism.ineq}
\begin{aligned}
 \mathbb{P}\big(\wt\theta_t \in \Theta^{lin,\opt} \hspace{1mm} | \hspace{1mm} \F^x_{t}, \wh E_t \big) 
  &\geq \min_{\wh\theta : \| \wh\theta - \theta_*\|_{V_t} \leq \beta_t(\delta^\prime)}   \mathbb{P}_{\wt\theta_t  \sim \mathcal{N}(\wh\theta, \beta_t^2(\delta^\prime) V_t^{-1})} \big(\wt\theta_t \in \Theta^{lin,\opt} \hspace{1mm} | \hspace{1mm} \F^x_{t}) \\
    &\geq \min_{\wh\theta : | (\wh\theta - \theta_*)^\transp H_* |\leq \beta_t(\delta^\prime)\|H_*\|_{V_{t}^{-1}} }   \mathbb{P}_{\wt\theta_t  \sim \mathcal{N}(\wh\theta, \beta_t^2(\delta^\prime) V_t^{-1})} \big(\wt\theta_t \in \Theta^{lin,\opt} \hspace{1mm} | \hspace{1mm} \F^x_{t}) \\
    &= \min_{\wh\theta : | \wh\theta^\transp H_* - \rho_*| \leq \beta_t(\delta^\prime)\|H_*\|_{V_{t}^{-1}} } \mathbb{P}_{\wt\theta_t  \sim \mathcal{N}(\wh\theta, \beta_t^2(\delta^\prime) V_t^{-1})} \big(| \wt\theta_t^\transp H_* | \leq \rho_* \hspace{1mm} | \hspace{1mm} \F^x_{t})
\end{aligned}
\end{equation}
Cor.~\ref{co:worst.case.ellipsoid} provides us with an explicit expression of the worst case ellipsoid. Introducing $x=\wt\theta_t^\transp H_*$, one has
$x \sim \mathcal{N}(\bar{x},\sigma^2_x)$ with $\bar{x} = \wh\theta H_*$ and $\sigma_x = \beta_t(\delta^\prime) \|H_* \|_{V_{t}^{-1}}$. Applying Cor.~\ref{co:worst.case.ellipsoid} with $\alpha = \rho_*$, $\rho = \rho_*$ and $\beta = \beta_t(\delta^\prime) \|H_*\|_{V_{t}^{-1}}$,  inequality~\eqref{eq:optimism.ineq} becomes
\begin{equation*}
\begin{aligned}
 \mathbb{P}\big(\wt\theta_t \in \Theta^{lin,\opt} \hspace{1mm} | \hspace{1mm} \F^x_{t}, \wh E_t \big) 
  &\geq \min_{\wh\theta : | \wh\theta^\transp H_* - \rho_*| \leq \beta_t(\delta^\prime) \|H_*\|_{V_t^{-1}}} \mathbb{P}_{\eta  \sim \mathcal{N}(0,I_2)} \big(| \wh\theta^\transp H_* + \beta_t(\delta^\prime) \eta^\transp V_{t}^{-1/2} H_* | \leq \rho_* \hspace{1mm} | \hspace{1mm} \F^x_{t})\\
 & \geq \mathbb{P}_{\eta  \sim \mathcal{N}(0,I_2)} \big(| \rho_* +  \beta_t(\delta^\prime) \|H_*\|_{V_{t}^{-1}} + \beta_t(\delta^\prime) \eta^\transp V_{t}^{-1/2} H_* | \leq \rho_* \hspace{1mm} | \hspace{1mm} \F^x_{t})
\end{aligned}
\end{equation*}
Introducing the vector $u_t =  \beta_t(\delta^\prime) V_{t}^{-1/2} H_*$, one can simplify
\begin{equation*}
\begin{aligned}
& | \rho_* +  \beta_t(\delta^\prime) \|H_*\|_{V_{t}^{-1}} + \beta_t(\delta^\prime) \eta^\transp V_{t}^{-1/2} H_* | \leq \rho_*, \\
\Leftrightarrow & -\rho_* \leq  \rho_* +  \|u_t\| + \eta^\transp u_t \leq \rho_*, \\
\Leftrightarrow & -\frac{\rho_*}{\|u_t\|} -1  \leq\eta^\transp \frac{u_t}{\|u_t\|}  \leq -1.
\end{aligned}
\end{equation*}
Since $\eta \sim \mathcal{N}(0,I_2)$ is rotationally invariant , $ \mathbb{P}\big(\wt\theta_t \in \Theta^{lin,\opt} \hspace{1mm} | \hspace{1mm} \F^x_{t}, \wh E_t \big) 
\geq  \mathbb{P}_{\epsilon \sim \mathcal{N}(0,1) } \big( \epsilon \in \big[1, 1+\frac{2 \rho_*}{\| u_t \|}\big] \hspace{1mm} | \hspace{1mm} \F^x_{t}, \wh E_t \big)$. Finally, for all $t\leq T$, $u_t$ is almost surely bounded: $\|u_t\| \leq \beta_T(\delta^\prime) \sqrt{(1+C^2)/\lambda}$. Therefore,
\begin{equation*}
\mathbb{P}\big(\wt\theta_t \in \Theta^{lin,\opt} \hspace{1mm} | \hspace{1mm} \F^x_{t}, \wh E_t \big)  \geq \mathbb{P}_{\epsilon \sim \mathcal{N}(0,1) } \big( \epsilon \in \big[1, 1+ 2 \rho_*/\beta_T(\delta^\prime) \sqrt{(1+C^2)/\lambda} \big] \big) := p
\end{equation*}

\begin{corollary}\label{co:worst.case.ellipsoid}
For any $\rho,\sigma_x> 0$, for any $\alpha, \beta \geq 0$, $\arg\min_{\bar{x} : | \bar{x} - \alpha | \leq \beta }  \mathbb{P}_{x \sim \mathcal{N}(\bar{x},\sigma_x^2)} \big( |x| \leq \rho \big) = \alpha  + \beta$.
\end{corollary}
This corollary is a direct consequence of the properties of standard gaussian r.v.

\begin{lemma}\label{le:gaussian.cdf.variation}
Let $x$ be a real random variable. For any $\rho, \sigma_x > 0$ Let $f : \mathbb{R} \rightarrow [0,1]$ be the continuous mapping defined by $f(\bar{x}) = \mathbb{P}_{x \sim \mathcal{N}(\bar{x},\sigma_x^2)} \big( |x| \leq \rho \big)$. Then, $f$ is increasing on $\mathbb{R}_{-}$ and decreasing on $\mathbb{R}_{+}$.
\end{lemma}
\begin{proof}
Without loss of generality, one can assume that $\sigma_x = 1/\sqrt{2}$ (otherwise, modify $\rho$), and that $\bar{x} \geq 0$ (by symmetry). Denoting as $\Phi$ and $\erf$ the standard gaussian cdf and the error function, one has:
\begin{equation*}
\begin{aligned}
f(\bar{x}) &=  \mathbb{P}_{x \sim \mathcal{N}(\bar{x},\sigma_x^2)} \big(   - \rho \leq x \leq \rho \big),
 = \mathbb{P}_{x \sim \mathcal{N}(\bar{x},\sigma_x^2)} \big( x \leq \rho \big)  - \mathbb{P}_{x \sim \mathcal{N}(\bar{x},\sigma_x^2)} \big( x \leq -\rho \big), \\
 &=  \mathbb{P}_{x \sim \mathcal{N}(\bar{x},\sigma_x^2)} \big( (x - \bar{x})/\sigma_x \leq (\rho - \bar{x})/\sigma_x \big)  - \mathbb{P}_{x \sim \mathcal{N}(\bar{x},\sigma_x^2)} \big( (x - \bar{x})/\sigma_x  \leq (-\rho - \bar{x})/\sigma_x  \big), \\
 &= \Phi (  (\rho - \bar{x})/\sigma_x)  - \Phi (  - (\rho + \bar{x})/\sigma_x), \\
 &= \frac{1}{2} + \frac{1}{2} \erf ( (\rho - \bar{x})/\sqrt{2} \sigma_x) -  \frac{1}{2} - \frac{1}{2} \erf ( -(\rho + \bar{x})/\sqrt{2} \sigma_x), \\
 &= \frac{1}{2} \big( \erf(\rho - \bar{x}) - \erf(-(\rho + \bar{x})) \big).
\end{aligned}
\end{equation*}
Since $\erf$ is odd, one obtains $f(\bar{x}) =  \frac{1}{2} \big( \erf(\rho - \bar{x}) + \erf(\rho + \bar{x}) \big)$. The error function is differentiable with $\erf^\prime(z) = \frac{2}{\pi} e^{-z^2}$, thus
\begin{equation*}
\begin{aligned}
f^\prime(\bar{x}) & = \frac{1}{\pi} \Big( \exp \big( - (\rho + \bar{x})^2 \big) - \exp \big( -(\rho - \bar{x})^2 \big) \Big) \\
&= - \frac{2}{\pi} \sinh \big( (\rho - \bar{x})^2 \big) \leq 0
\end{aligned}
\end{equation*}
Hence, $f$ is decreasing on $\mathbb{R}_+$ and by symmetry, is increasing on $\mathbb{R}_-$.
\end{proof}

\end{enumerate}

% !TEX root = main.tex

\section{Weighted L1 Poincaré inequality (proof of Lem.~\ref{th:poincare.weigthed})}\label{sec:app_poincare}
This result is build upon the following theorem which links the function to its gradient in $L^1$ norm:
\begin{theorem}[see ~\citet{acosta2004optimal}]\label{th:poincare_ineq_L1}
Let $W^{1,1}(\Omega)$ be the Sobolev space on $\Omega \subset \mathbb{R}^d$. Let $\Omega$ be a convex domain bounded with diameter $D$ and $f \in W^{1,1}(\Omega)$  of zero average on $\Omega$ then 
\begin{equation}\label{eq:poincare_ineq_L1}
\int_{\Omega} |f(x)| dx \leq \frac{D}{2} \int_{\Omega} || \nabla f(x) || dx
\end{equation}
\end{theorem}
Lem.~\ref{th:poincare.weigthed} is an extension of Thm.~\ref{th:poincare_ineq_L1}. In pratice, we show that their proof still holds for log-concave weight.
\begin{theorem}
\label{th:1D_weighed_poincare}
Let $L > 0$ and $\rho$ any non negative and log-concave function on $[0,L]$. Then for any $f \in W^{1,1}(0,L)$ such that 
\begin{equation*}
\int_{0}^L f(x) \rho(x) dx = 0 
\end{equation*}
one has:
\begin{equation}
\int_{0}^L |f(x)| \rho(x) dx \leq 2 L \int_{0}^L |f^\prime(x)| \rho (x) dx
\label{eq:1D_weighed_poincare}
\end{equation}
\end{theorem}
The proof is based on the following inequality for log-concave function.
\begin{lemma}
\label{le:logconcave_inequality}
Let $\rho$ be any non negative log-concave function on $[0,1]$ such that $\int_0^1 \rho(x) = 1$ then 
\begin{equation}
\forall x  \in (0,1), \hspace{3mm} H(\rho,x) := \frac{1}{\rho(x)} \int_0^x \rho(t) dt \int_x^1 \rho(t) dt \leq 1
\label{eq:logconcave_inequality}
\end{equation}
\end{lemma}
\begin{proof}
Since any non-negative log-concave function on $[0,1]$ can be rewritten as $\rho(x) = e^{\nu(x)}$ where $\nu$ is a concave function on $[0,1]$ and since $x \rightarrow e^x$ is increasing, the monotonicity of $\nu$ is preserved and as for concave function,  
$\rho$ can be either increasing, decreasing or increasing then decreasing on $[0,1]$.\\
Hence, $\forall x \in (0,1)$, either 
\begin{enumerate}
\item $\rho(t) \leq \rho(x)$ for all $t \in [0,x]$,
\item $\rho(t) \leq \rho(x)$ for all $t \in [x,1]$.
\end{enumerate}

Assume that $\rho(t) \leq \rho(x)$ for all $t \in [0,x]$ without loss of generality. Then,
\begin{equation*}
\begin{split}
\forall x  \in (0,1), \hspace{3mm} H(\rho,x) &:= \frac{1}{\rho(x)} \int_0^x \rho(t) dt \int_x^1 \rho(t) dt  \\
&= \int_0^x \frac{\rho(t)}{\rho(x)} \int_x^1 \rho(t) dt \\
&\leq \int_0^x dt \int_x^1 \rho(t) dt \\
&\leq x \int_0^1 \rho(t) dt \leq x \leq 1
\end{split}
\end{equation*}
\end{proof}

\begin{proof}[Proof of theorem \ref{th:1D_weighed_poincare}]
This proof is exactly the same as \cite{acosta2004optimal} where we use lemma \ref{le:logconcave_inequality} instead of a concave inequality. We provide it for sake of completeness.\\

A scaling argument ensures that it is enough to prove it for $L = 1$. Moreover, dividing both side of \eqref{eq:1D_weighed_poincare}
by $\int_0^1 \rho(x)dx$, we can assume without loss of generality that $\int_0^1 \rho(x) dx = 1$.\\
Since $\int_0^1 f(x) \rho(x) dx = 0$ by integration part by part one has:
\begin{equation*}
\begin{split}
f(y) & = \int_0^y f^\prime(x) \int_0^x \rho(t) dt - \int_y^1 f^\prime(x) \int_x^1 \rho(t) dt \\
|f(y)| &\leq \int_0^y |f^\prime(x)| \int_0^x \rho(t) dt + \int_y^1 |f^\prime(x)| \int_x^1 \rho(t) dt
\end{split}
\end{equation*}
Multiplying by $\rho(y)$, integrating on $y$ and applying Fubini's theorem leads to 
\begin{equation*}
\int_0^1 |f(y)| \rho(y) dy \leq 2 \int_0^1 | f^\prime(x)| \int_0^x \rho(t) dt \int_x^1 \rho(t) dt
\end{equation*}
and applying \eqref{eq:logconcave_inequality} of lemma \ref{le:logconcave_inequality} ends the proof.
\end{proof}
While theorem \ref{th:1D_weighed_poincare} provides a 1 dimensional weigthed Poincaré inequality, we actually seek for one in $\mathbb{R}^d$. The idea of \cite{acosta2004optimal} is to use arguments of~\cite{payne1960optimal} to reduce the $d-$dimensional problem to a $1-d$ problem by splitting any convex set $\Omega$ into subspaces $\Omega_i$ thin in all but one direction and such that an average property is preserved. We just provide their result.
\begin{lemma}
\label{le:convex_decomposition}
Let $\Omega \subset \mathbb{R}^d$ be a convex domain with finite diameter $D$ and $u \in L^1(\Omega)$ such that $\int_\Omega u = 0$. Then, for any $\delta > 0$, there exists a decomposition of $\Omega$ into a finite number of convex domains $\Omega_i$ satisfying
\begin{equation*}
\Omega_i \cap \Omega_j = \emptyset \hspace{2mm} \text{for} \hspace{2mm}  i \neq j, \hspace{3mm} \bar{\Omega} = \bigcup \bar{\Omega}_i, \hspace{3mm} \int_{\Omega_i} u = 0
\end{equation*}
and each $\Omega_i$ is thin in all but one direction i.e. in an appropriate rectangular coordinate system $(x,y) = (x,y_1,\dots,y_{d-1})$ the set $\Omega_i$ is contained in 
\begin{equation*}
\{ (x,y) : \hspace{2mm} 0 \leq x \leq D, \hspace{3mm} 0 \leq y_i \leq \delta \hspace{2mm} \text{for} \hspace{2mm} i = 1,\dots,d-1 \}
\end{equation*}
\end{lemma}

This decomposition together with theorem \ref{th:1D_weighed_poincare} allow us to prove the $d-$dimensional weighted Poincaré inequality.
\begin{proof}[Proof of Lem.~\ref{th:poincare.weigthed}]
By density, we can assume that $u \in C^\infty(\bar{\Omega})$. Hence, $u p \in C^2(\bar{\Omega})$. Let $M$ be a bound for $u p$ and all its derivative up to the second order.\\
Given $\delta > 0$ decompose the set $\Omega$ into $\Omega_i$ as in lemma \ref{le:convex_decomposition} and express $z \in \Omega_i$ into the appropriate rectangular basis $z =(x,y)$, where $x \in [0,d_i]$, $y \in [0,\delta]$. Define as $\rho(x_0)$ the $d-1$ volume of the intersection between $\Omega_i$ and the hyperplan $\{x = x_0\}$. Since $\Omega_i$ is convex, $\rho$ is concave and from the smoothness of $up$ one has:

\begin{align}
& \left| \int_{\Omega_i} | u(x,y) | p(x,y) dx dy - \int_{0}^{d_i} |u(x,0)| p(x,0) \rho(x) dx \right| \leq (d-1) M |\Omega_i| \delta \label{eq:ineq_smooth_1}\\
&\left| \int_{\Omega_i} | \frac{\partial u}{\partial x}(x,y) | p(x,y) dx dy - \int_{0}^{d_i} |\frac{\partial u}{\partial x} (x,0)| p(x,0) \rho(x) dx \right| \leq (d-1) M |\Omega_i| \delta  \label{eq:ineq_smooth_2} \\
& \left| \int_{\Omega_i}  u(x,y)  p(x,y) dx dy - \int_{0}^{d_i} u(x,0) p(x,0)  \rho(x) dx \right| \leq (d-1) M |\Omega_i| \delta \label{eq:ineq_smooth_3}
\end{align}

Those equation allows us to switch from $d-$dimensional integral to $1-$dimensional integral for which we can apply theorem \ref{th:1D_weighed_poincare} at the condition that $\int_0^{d_i} u(x,0) p(x,0) \rho(x) dx = 0$ (which is not satisfied here). On the other hand, we can apply theorem \ref{th:1D_weighed_poincare} to 
\begin{equation*}
g(x) = u(x,0) - \int_{0}^{d_i} u(x,0) p(x,0) \rho(x) dx  / \int_0^{d_i} p(x,0) \rho(x) dx
\end{equation*}
with weigthed function $x \rightarrow p(x,0) \rho(x)$. Indeed, $x \rightarrow p(x,0)$ is log-concave - as restriction along one direction of log-concave function, $x \rightarrow \rho(x)$ is log-concave - as a concave function, and so is $x \rightarrow p(x,0) \rho(x)$ - as product of log-concave function. Moreover, $g \in W^{1,1}(0,d_i)$ and $\int_0^{d_i} g(x) p(x,0) \rho(x) dx = 0$ by construction. Therefore, applying theorem \ref{th:1D_weighed_poincare} one gets:

\begin{equation}
\begin{split}
& \int_{0}^{d_i} | g(x) | p(x,0) \rho(x) dx \leq 2 d_i \int_{0}^{d_i} |g^\prime (x) |p(x,0) \rho(x) dx  \\
& \int_{0}^{d_i} |u(x,0)| p(x,0) \rho(x) dx \leq 2 d_i \int_{0}^{d_i} |\frac{\partial u}{\partial x}(x,0)| p(x,0) \rho(x) dx - 
\left | \int_{0}^{d_i} u(x,0) p(x,0) \rho(x) dx \right| \\
& \int_{0}^{d_i} |u(x,0)| p(x,0) \rho(x) dx \leq 2 d_i \int_{0}^{d_i} |\frac{\partial u}{\partial x}(x,0)| p(x,0) \rho(x) dx + (d-1) M |\Omega_i| \delta
\end{split}
\label{eq:ineq_smooth_4}
\end{equation}
where we use equation \eqref{eq:ineq_smooth_3} together with $\int_{\Omega_i} u(z) p(z) dz = 0$ to obtain the last inequality.\\

Finally, from \eqref{eq:ineq_smooth_1}
\begin{equation*}
\int_{\Omega_i} | u(x,y) | p(x,y) dx dy \leq \int_{0}^{d_i} |u(x,0)| p(x,0) \rho(x) dx  +  (d-1) M |\Omega_i| \delta
\end{equation*}
from \eqref{eq:ineq_smooth_4}
\begin{equation*}
\int_{\Omega_i} | u(x,y) | p(x,y) dx dy \leq 2 d_i \int_{0}^{d_i} |\frac{\partial u}{\partial x}(x,0)| p(x,0) \rho(x) dx  +  (d-1) M |\Omega_i| \delta (1 + 2 d_i)
\end{equation*}
from \eqref{eq:ineq_smooth_2}
\begin{equation*}
\begin{split}
&\int_{\Omega_i} | u(x,y) | p(x,y) dx dy \leq 2 d_i \int_{\Omega_i} | \frac{\partial u}{\partial x}(x,y) | p(x,y) dx dy  +  (d-1) M |\Omega_i| \delta (1 + 4 d_i) \\
&\int_{\Omega_i} | u(x,y) | p(x,y) dx dy \leq 2 d_i \int_{\Omega_i} || \nabla u(x,y) || p(x,y) dx dy +  (d-1) M |\Omega_i| \delta (1 + 4 d_i) \\
\end{split} 
\end{equation*}
Summing up on $\Omega_i$ leads to
\begin{equation*}
\int_\Omega |u(z)|p(z) dz \leq 2 D \int_\Omega || \nabla u(z) || p(z) dz + (d-1) M | \Omega | \delta (1 + 4 D)
\end{equation*}
and since $\delta$ is arbitrary one gets the desired result.
\end{proof}

% !TEX root = main.tex

\section{Regret proofs}\label{sec:app_regret_proofs}
\textbf{Bounding $R^\rls_1$.} On $E$, $\|x_t\| \leq X$ for all $t \in [0,T]$. Moreover, since $\wt\theta_t \in \mathcal{S}$ for all $t\in[0,T]$ due to the rejection sampling, $\Tr (P(\wt \theta_t)) \leq D$.
 From the definition of the matrix 2-norm, $ \sup_{\|x\| \leq X} x^\transp P(\wt \theta_t) x \leq X^2 \| P(\wt \theta_t)^{1/2} \|^2_2$. Since for any $A \in \mathbb{R}^{m,n}$, $\|A\|_2 \leq \|A\| $, one has $\| P(\wt \theta_t)^{1/2} \|^2_2 \leq \| P(\wt \theta_t)^{1/2} \|^2 = \Tr P(\wt \theta_t)$. As a consequence, for any $t \in [0,T]$, $ \sup_{\|x\| \leq X} x^\transp P(\wt \theta_t) x \leq X^2 D$ and the martingale increments are bounded almost surely on $E$ by $2 D X^2$.\\
 Applying Thm.~\ref{th:azuma.ineq} to $R^\rls_1$ with $\epsilon = 2 D X^2 \sqrt{2 T \log(4/\delta)}$ one obtains that
 \begin{equation*}
 R^\rls_1  = \sum_{t=0}^T \big\{ \mathbb{E} ( x_{t+1}^\transp P(\wt\theta_{t+1}) x_{t+1}|\mathcal{F}_{t}) - x_t^\transp P(\wt\theta_t) x_t  \big\} \I \{E_t \} \leq  2 D X^2 \sqrt{2 T \log(4/\delta)}
\end{equation*}
with probability at least $1 - \delta/2$.

\textbf{Bounding $R^\rls_3$.} The derivation of this bound is directly collected from~\citet{abbasi2011regret}. Since our framework slightly differs, we provide it for the sake of completeness. The whole derivation is performed conditioned on the event $E$.\\
\begin{equation*}
\begin{aligned}
R^\rls_3 &= \sum_{t=0}^T \big\{ z_t^\transp \wt\theta_t P(\wt\theta_t) \wt\theta_t^\transp z_t  - z_t^\top \theta_* P(\wt\theta_t) \theta_*^\transp z_t  \big\} 
= \sum_{t=0}^T \big\{ \| \wt\theta_t^\transp z_t \|^2_{P(\wt\theta_t)}  -  \| \theta_*^\transp z_t \|^2_{P(\wt\theta_t)}  \big\}, \\
&= \sum_{t=0}^T \big( \| \wt\theta_t^\transp z_t \|_{P(\wt\theta_t)}  -  \| \theta_*^\transp z_t \|_{P(\wt\theta_t)} \big) 
\big( \| \wt\theta_t^\transp z_t \|_{P(\wt\theta_t)}  +  \| \theta_*^\transp z_t \|_{P(\wt\theta_t)} \big) \\
\end{aligned}
\end{equation*}
By the triangular inequality, $\| \wt\theta_t^\transp z_t \|_{P(\wt\theta_t)}  -  \| \theta_*^\transp z_t \|_{P(\wt\theta_t)} \leq \| P(\wt\theta_t)^{1/2} ( \wt\theta_t^\transp z_t -  \theta_*^\transp z_t ) \| \leq \|P(\wt\theta_t) \| \| (\wt\theta_t^\transp -  \theta_*^\transp)  z_t \| $. Making use of the fact that $\wt\theta_t \in \mathcal{S}$ by construction of the rejection sampling, $\theta_\star \in \mathcal{S}$ by Asm.~\ref{asm:control.asm} and that $\sup_{t\in[0,T]} \| z_t\| \leq \sqrt{ (1 + C^2) X^2 }$ thanks to the conditioning on $E$ and Prop.~\ref{prop:admissib_param_set}, one gets:
\begin{equation*}
R^\rls_3 \leq  \sum_{t=0}^T \big( \sqrt{D} \| (\wt\theta_t^\transp -  \theta_*^\transp)  z_t \| \big) \big( 2 S \sqrt{D} \sqrt{ (1 + C^2) X^2 } \big) 
\leq 2 S D \sqrt{ (1 + C^2) X^2 } \sum_{t=0}^T  \| (\wt\theta_t^\transp -  \theta_*^\transp)  z_t \| 
\end{equation*}
and one just has to bound $\sum_{t=0}^T  \| (\wt\theta_t^\transp -  \theta_*^\transp)  z_t \|$. Let $\tau(t) \leq t$ be the last time step before $t$ when the parameter was updated. Using Cauchy-Schwarz inequality, one has:
\begin{equation*}
\sum_{t=0}^T  \| (\wt\theta_t^\transp -  \theta_*^\transp)  z_t \|  = \sum_{t=0}^T  \| (V^{1/2}_\tau(t) (\wt\theta_{\tau(t)} -  \theta_*))^\transp  V_{\tau(t)}^{-1/2} z_t \|  
\leq \sum_{t=0}^T \|\wt\theta_{\tau(t)} -  \theta_*\|_{V_{\tau(t)}} \|z_t \|_{V_{\tau(t)}^{-1}}
\end{equation*}
However, on $E$, $\|\wt\theta_{\tau(t)} -  \theta_*\|_{V_{\tau(t)}} \leq \|\wt\theta_{\tau(t)} -  \wh\theta_{\tau(t)}\|_{V_{\tau(t)}}  + \|\theta_* -  \wh\theta_{\tau(t)}\|_{V_{\tau(t)}} \leq \beta_{\tau(t)}(\delta^\prime) + \gamma_{\tau(t)}(\delta^\prime) \leq  \beta_T(\delta^\prime) + \gamma_T(\delta^\prime)$ and, thanks to the lazy update rule $\|z_t \|_{V_{\tau(t)}^{-1}} \leq \|z_t \|_{V_{t}^{-1}} \frac{\det(V_t)}{\det(V_{\tau(t)})} \leq 2  \|z_t \|_{V_{t}}$. Therefore, 
\begin{equation*}
R^\rls_3 \leq 4 S D \sqrt{ (1 + C^2) X^2 } \big( \beta_T(\delta^\prime) + \gamma_T(\delta^\prime) \big) \sum_{t=0}^T \|z_t \|_{V_t^{-1}}.
\end{equation*}

\textbf{Bounding $\sum_{k=1}^K T_k \alpha_k$.} From section~\ref{subsec:final.bound}, 
 \begin{equation*}
 \sum_{k=1}^K T_k \alpha_k \leq 2 \tau  \sum_{k \in \mathcal{K}^{den}}  \| \alpha_k \| + \tau \sum_{k=1}^K \alpha_k.
 \end{equation*}
First, it is clear from 
\begin{equation*}
\begin{aligned}
\alpha_k &= (R^{\ts,1}_{t_k} + R^{\ts,3}_{t_k} )\{E_{t_k} \} \\
&= \big(J(\wt\theta_{t_k})- \mathbb{E}[J(\wt\theta_{t_k})| \mathcal{F}^x_{t_k},E_{t_k}] \big) \I\{E_{t_k}\}, 
+ \big( \mathbb{E} \Big[ \Big \|  \begin{pmatrix} I \\ K(\wt\theta_{t_k})^\top \end{pmatrix} \Big \|_{V_{t_k}^{-1}}|\mathcal{F}^x_{t_k} \Big] -  \Big \|  \begin{pmatrix} I \\ K(\wt\theta_{t_k})^\top \end{pmatrix} \Big \|_{V_{t_k}^{-1}} \big),
\end{aligned}
\end{equation*}
that the sequence $\{\alpha_k\}_{k=1}^K$ is a martingale difference sequence with respect to $\mathcal{F}^x_{t_k}$. Moreover, since $\wt\theta_{t_k} \in \mathcal{S}$ for all $k\in [1,K]$, $\| \alpha_k\| \leq 2 D + 2 \sqrt{ (1 + C^2 ) /\lambda}$. Therefore,
\begin{enumerate}
\item $ \sum_{k \in \mathcal{K}^{den}}  \| \alpha_k \| \leq \big( 2 D + 2 \sqrt{ (1 + C^2 )} \big) |K^{den}|$,
\item with probability at least $1-\delta/2$, Azuma's inequality ensures that $ \sum_{k=1}^K \alpha_k \leq  \big( 2 D + 2 \sqrt{ (1 + C^2 )} \big) \sqrt{2 |K| \log (4/\delta)}$.
\end{enumerate}
From Lem.~\ref{le:det.update.number.bound} and Cor.~\ref{co:update.number.bound}, $|K^{det}| \leq (n+d) \log_2 ( 1 + T X^2 (1+C^2)/ \lambda)$ and $|K| \leq (n+d) \log_2 ( 1 + T X^2 (1+C^2)/ \lambda) T/\tau$. Finally, one obtains:
\begin{equation*}
\sum_{k=1}^K T_k \alpha_k \leq    4 \big( 2 D + 2 \sqrt{ (1 + C^2 )} \big)  (n+d) \log_2 ( 1 + T X^2 (1+C^2)/ \lambda) \sqrt{ \log (4/\delta)}T/\tau
\end{equation*}

\textbf{Bounding $\sum_{t=0}^T \|z_t \|_{V_t^{-1}}$.} On $E$, for all $t \in [0,T]$, $\|z_t\|^2 \leq (1 + C^2) X^2$. Thus, from Cauchy-Schwarz inequality and Prop.~\ref{p:self_normalized_determinant},
\begin{equation*}
\sum_{t=0}^T \|z_t \|_{V_t^{-1}} \leq \sqrt{T} \Big( \sum_{t=0}^T \|z_t \|^2_{V_t^{-1}} \Big)^{1/2} \leq \sqrt{T} \sqrt{2(n + d) (1 + C^2) X^2 /\lambda} \log^{1/2}\Big( 1 + \frac{T (1 + C^2) X^2}{\lambda(n+d)}\Big).
\end{equation*}

\end{document}